\documentclass{article}

\usepackage{arxiv}

\usepackage[utf8]{inputenc} 
\usepackage[T1]{fontenc}    
\usepackage{hyperref}       
\usepackage{url}            
\usepackage{booktabs}       
\usepackage{amsfonts}       
\usepackage{nicefrac}       
\usepackage{microtype}      
\usepackage{lipsum}

\usepackage{amsmath}
\usepackage{algorithm}
\usepackage{algorithmic}
\usepackage{comment}
\usepackage{amsfonts}
\usepackage{amsthm}
\usepackage{bbm}
\usepackage{dsfont}
\usepackage[utf8]{inputenc}
\usepackage[english]{babel}

\usepackage{graphicx,color}
\usepackage{caption}
\usepackage{subcaption}
\newtheorem{theorem}{Theorem}[section]
\newtheorem{lemma}[theorem]{Lemma}

\title{Observe Before Play: Multi-armed Bandit with Pre-observations}

\author{%
  Jinhang Zuo, Xiaoxi Zhang, Carlee Joe-Wong\\
  Carnegie Mellon University\\
\{jzuo,xiaoxiz2,cjoewong\}@andrew.cmu.edu \\
}

\begin{document}
\maketitle

\begin{abstract}
    We consider the stochastic multi-armed bandit (MAB) problem in a setting where a player can pay to pre-observe arm rewards before playing an arm in each round. Apart from the usual trade-off between exploring new arms to find the best one and exploiting the arm believed to offer the highest reward, we encounter an additional dilemma: pre-observing more arms gives a higher chance to play the best one, but incurs a larger cost. For the single-player setting, we design an Observe-Before-Play Upper Confidence Bound (OBP-UCB) algorithm for $K$ arms with Bernoulli rewards, and prove a $T$-round regret upper bound $O(K^2\log T)$. In the multi-player setting, collisions will occur when players select the same arm to play in the same round. We design a centralized algorithm, C-MP-OBP, and prove its $T$-round regret relative to an offline greedy strategy is upper bounded in $O(\frac{K^4}{M^2}\log T)$ for $K$ arms and $M$ players. We also propose distributed versions of the C-MP-OBP policy, called D-MP-OBP and D-MP-Adapt-OBP, achieving logarithmic regret with respect to collision-free target policies. Experiments on synthetic data and wireless channel traces show that C-MP-OBP and D-MP-OBP outperform random heuristics and offline optimal policies that do not allow pre-observations.
\end{abstract}
\section{Introduction}\label{sec:intro}
Multi-armed bandit (MAB) problems have attracted much attention as a means of capturing the trade-off between exploration and exploitation \cite{bubeck2012regret} in sequential decision making. In the classical MAB problem, a player chooses one of a fixed set of arms and receives a reward based on this choice. The player aims to maximize her cumulative reward over multiple rounds, navigating a tradeoff between exploring unknown arms (to potentially discover an arm with higher rewards) and exploiting the best known arm (to avoid arms with low rewards). Most MAB algorithms use the history of rewards received from each arm to design optimized strategies for choosing which arm to play. They generally seek to prove that the \emph{regret}, or the expected difference in the reward compared to the optimal strategy when all arms' reward distributions are known in advance, grows sub-linearly with the number of rounds.

\subsection{Introducing Pre-observations}
The classical MAB exploration-exploitation tradeoff arises because knowledge about an arm's reward can only be obtained by playing that arm. In practice, however, this tradeoff may be relaxed. \cite{Yun:MAB}, for example, suppose that at the end of each round, the player can pay a cost to observe the rewards of additional un-played arms, helping to find the best arm faster. In cascading bandits~\cite{kveton2015cascading}, players may choose multiple arms in a single round, e.g., if the ``arms'' are search results in a web search application.

In both examples above, the observations made in each round do not influence the choice of arms in that round. In this paper, we introduce the \textbf{MAB problem with pre-observations}, where in each round, the player can pay to pre-observe the realized rewards of some arms before choosing an arm to play. For instance, one might play an arm with high realized reward as soon as it is pre-observed. 
Pre-observations can help to reconcile the exploration-exploitation tradeoff, but they also introduce an additional challenge: namely, {\it \textbf{optimizing the order of the pre-observations}}.
This formulation is inspired by Cognitive Radio Networks (CRNs), where users can use wireless channels when they are unoccupied by primary users. In each round, a user can sense (pre-observe) some channels (arms) to check their availability (reward) before choosing a channel to transmit data (play). Sensing more arms leaves less time for data transmission, inducing a cost of making pre-observations. 

In this pre-observation example, there are negative network effects when multiple players attempt to play the same arm: if they try to use the same wireless channel, for instance, the users ``collide'' and all transmissions fail. 
In multi-player bandit problems without pre-observations, players generally minimize these collisions by allocating themselves so that each plays a distinct arm with high expected reward. In our problem, the players must instead learn \emph{ordered sequences} of arms that they should pre-observe, minimizing overlaps in the sequences that might induce players to play the same arm. Thus, one user's playing a sub-optimal arm may affect other users' pre-observations, leading to cascading errors. We then encounter a new challenge of {\it \textbf{designing users' pre-observation sequences}} to minimize collisions but still explore unknown arms. This problem is particularly difficult when {\it \textbf{players cannot communicate or coordinate with each other}} to jointly design their observation sequences.
To the best of our knowledge, such \emph{multi-player bandit problems with pre-observations} have not been studied in the literature. 
\subsection{Applications}
Although many MAB works take cognitive radios as their primary motivation \cite{rosenski2016multi,besson2018multi,kumar2018trekking}, multi-player bandits with pre-observations could be applied to any scenario where users search for sufficiently scarce resources at multiple providers that are either acceptable (to all users) or not. We briefly list three more applications. First, users may sequentially bid in auctions (arms) offering equally useful items, e.g., Amazon EC2 spot instance auctions for different regions, stopping when they win an auction. Since these resources are scarce, each region may only be able to serve one user (modeling collisions between users). Second, in distributed caching, each user (player) may sequentially query whether one of several caches (arms) has the required file (is available), but each cache can only send data to one user at a time (modeling collisions). Third, taxis (players) can sequentially check locations (arms) for passengers (availability); collisions occur since each passenger can only take one taxi, and most locations (e.g., city blocks that are not next to transit hubs) would not have multiple passengers looking for a taxi at the same time. 

\subsection{Our Contributions} 
Our first contribution is to {\bf develop an Observe-Before-Play (OBP) policy} to maximize the total reward of a single user via minimizing the cost spent on pre-observations. Our OBP policy achieves a regret bound that is logarithmic with time and quadratic in the number of available arms. It is consistent with prior results~\cite{li2014almost}, and more easily generalizes to multi-player settings. 
In the rest of the paper, ``user'' and ``player'' are interchangeable. 

We next consider the multi-player setting. Unlike in the single-player setting, it is not always optimal to observe the arms with higher rewards first. We show that finding the offline optimal policy to maximize the overall reward of all players is NP-hard.
However, we give conditions under which a greedy allocation that avoids user collisions is offline-optimal; in practice, this strategy performs well. Our second research contribution is then to {\bf develop a centralized C-MP-OBP policy} that generalizes the OBP policy for a single user. Despite the magnified loss in reward when one user observes the wrong arm, we show that the C-MP-OBP policy can learn the arm rankings, and that its regret relative to the offline greedy strategy is logarithmic with time and polynomial in the number of available arms and users.
Our third research contribution is to {\bf develop distributed versions of our C-MP-OBP policy, called D-MP-OBP and D-MP-Adapt-OBP}. Both algorithms assume no communication between players and instead use randomness to avoid collisions. Despite this lack of communication, both achieve logarithmic regret over time with respect to the collision-free offline greedy strategies defined in the centralized setting. 

Our final contribution is to {\bf numerically validate our OBP, C-MP-OBP, and D-MP-OBP policies on synthetic reward data and channel availability traces}. We show that all of these policies outperform both random heuristics and traditional MAB algorithms that do not allow pre-observations, and we verify that they have sublinear regret over time. We further characterize the effect on the achieved regret of varying the pre-observation cost and the distribution of the arm rewards. 


We discuss related work in Section~\ref{sec:related} and consider the single-player setting in Section~\ref{sec:single}. We generalize these results to multiple players in centralized (Section~\ref{sec:centralized}) and distributed (Section~\ref{sec:distributed}) settings. We numerically validate our results in Section~\ref{sec:exp} and conclude in Section~\ref{sec:discuss}. Proofs are in Appendix.

\section{Related Work}\label{sec:related}
Multi-armed Bandit (MAB) problems have been studied since the 1950s~\cite{Lai1985,bubeck2012regret}. \cite{UCB1_survey}, for instance, propose a simple UCB1 policy that achieves logarithmic regret over time. 
Recently, MAB applications to Cognitive Radio Networks (CRNs) have attracted attention~\cite{ahmad2009optimality,lai2011cognitive}, especially in multi-player settings~\cite{liu2010distributed,distributed_jsac,avner2016multi,bonnefoi2017multi,kumar2018trekking} where users choose from the same arms (wireless channels). None of these works include pre-observations, though some~\cite{avner2014concurrent,rosenski2016multi,besson2018multi} consider distributed settings. \cite{li2014almost,combes2015learning} study the single-player MAB problem with pre-observations, but do not consider multi-player settings.

The proposed MAB with pre-observations in a single-player setting is a variant on cascading bandits~\cite{kveton2015cascading,kveton2015combinatorial,zong2016cascading}. The idea of pre-observations with costs is similar to the cost-aware cascading bandits proposed in \cite{zhou2018cost} and contextual combinatorial cascading bandits introduced in \cite{li2016contextual}. However, in \cite{zhou2018cost}, the reward collected by the player can be negative if all selected arms have zero reward in one round; in our model, the player will get zero reward if all selected arms are unavailable. Moreover, most cascading bandit algorithms are applied to recommendation systems, where there is only a single player. To the best of our knowledge, we are the first to study MAB problems with pre-observations in multi-player settings.

\section{Single-player Setting}\label{sec:single}
\begin{figure}[t]
\centering
\includegraphics[width=0.6\textwidth]{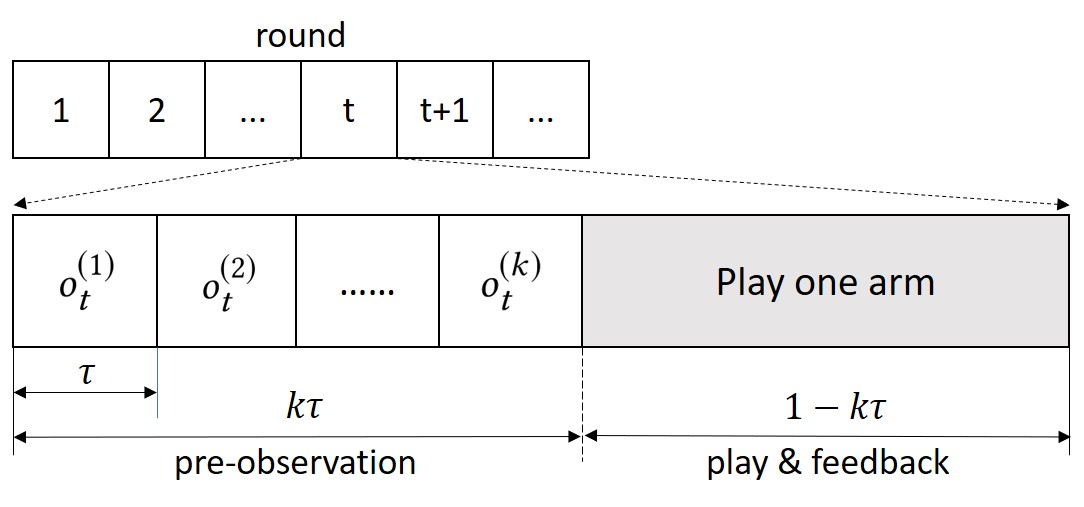}
\caption{Illustration of Pre-observations}
\label{fig:sensing}
\end{figure}
We consider a player who can pre-observe a subset of $K$ arms and play one of them, with a goal of maximizing the total reward over $T$ rounds. Motivated by the CRN scenario, we assume as in~\cite{distributed_jsac} an i.i.d. Bernoulli reward of each arm to capture the occupancy/vacancy of each channel (arm). Let $Y_{k,t}\stackrel{iid}{\sim} \text{Bern}(\mu_k) \in \{0,1\}$ denote the reward of arm $k$ at round $t$, with expected value $\mu_k \in [0,1]$. As shown in Figure \ref{fig:sensing}, in each round, the player chooses a pre-observation list $\boldsymbol{o}_t := (o^{(1)}_{t},o^{(2)}_{t},\dots,o^{(K)}_{t}),$ where $o^{(i)}_{t}$ represents the $i^{\text{th}}$ arm to be observed at $t$ and $\boldsymbol{o}_t$ is a permutation of $(1, 2, \dots, K)$. The player observes from the first arm $o^{(1)}_{t}$ to the last arm $o^{(K)}_{t}$, stopping at and playing the first good arm (reward = 1) until the list exhausts. We denote the index of the last observed arm in $\boldsymbol{o}_t$ as $I(t)$, which is the first available arm in $\boldsymbol{o}_t$ or $K$ if no arms are available.
Pre-observing each arm induces a constant cost $\tau$; in CRNs, this represents a constant time $\tau$ for sensing each channel's occupancy. We assume for simplicity that $0 < K\tau < 1$. The payoff received by the player at $t$ then equals: $(1-I(t)\,\tau) Y_{o_{t}^{(I(t))},t}$; if all the arms are bad (reward = 0) in round $t$, then the player will get zero reward for any $\boldsymbol{o}_t$. Given $\{\boldsymbol{o}_t\}^{T}_{t=1}$, we can then define the total realized and expected rewards received by the player in $T$ rounds:
\begin{align}
    r(T) &:= \sum_{t=1}^{T}(1-I(t)\,\tau) Y_{o_{t}^{(I(t))},t} \\
    \mathbb{E}[r(T)] &= \sum_{t=1}^{T}\sum_{k=1}^{K}\left\{(1 - k\,\tau)\mu_{o^{(k)}_{t}}\prod_{i=1}^{k-1}(1-\mu_{o^{(i)}_{t}})\right\},\label{eq:r(T)_single}
\end{align}
where $\prod_{i=1}^{0}(1-\mu_{o^{(i)}_{t}}):=1$.
We next design an algorithm for choosing $\boldsymbol{o}_t$ at each round $t$ to maximize $\mathbb{E}[r(T)]$.
We assume $\mu_1 \geq \mu_2 \geq \dots \geq \mu_K$ without loss of generality and first establish the optimal offline policy:
\begin{lemma}\label{lemma:descend_order}
The optimal offline policy $\boldsymbol{o}_t^{*}$ that maximizes the expected total reward is observing arms in the descending order of their expected rewards, {\em i.e.}, $\boldsymbol{o}_t^{*}=(1, 2, \dots, K)$.
\end{lemma}

\begin{algorithm}
 \caption{Observe-Before-Play UCB (OBP-UCB)}
 \begin{algorithmic}[1]
 \label{alg:single-obp}
 \STATE \textbf{Initialization}: Pull all arms once and update $n_{i}(t)$, $\overline{\mu}_{i}(t)$, $\hat{\mu}_{i}(t)$ for all $i \in [K]$ 
 \WHILE{$t$}
    \STATE $\boldsymbol{o}_t = \text{argsort}(\hat{\mu}_{1}(t), \hat{\mu}_{2}(t), \dots, \hat{\mu}_{K}(t))$;
    \FOR{$i = 1:K$}
        \STATE Observe arm $o_t^{(i)}$'s reward $Y_{o_t^{(i)},t}$;
        \STATE $n_{o_t^{(i)}}(t+1) = n_{o_t^{(i)}}(t) + 1$;
        \STATE $\overline{\mu}_{o_t^{(i)}}(t+1) = (\overline{\mu}_{o_t^{(i)}}(t) n_{o_t^{(i)}}(t) + Y_{o_t^{(i)},t}) / n_{o_t^{(i)}}(t+1)$;
        \IF{$Y_{o_t^{(i)},t} = 1$}
            \STATE Play arm $i$ for this round;
            \STATE $n_{o_t^{(j)}}(t+1) = n_{o_t^{(j)}}(t)$ for all $j > i$;
            \STATE $\overline{\mu}_{o_t^{(j)}}(t+1) = \overline{\mu}_{o_t^{(j)}}(t)$ for all $j > i$;
            \STATE break;
        \ENDIF
    \ENDFOR
    \STATE Update $\hat{\mu}_{i}(t)$ for all $i \in [K]$;
    \STATE $t = t+1$;
 \ENDWHILE
 \end{algorithmic} 
\end{algorithm}

Given this result, we propose an UCB (upper confidence bound)-type online algorithm, Observe-Before-Play UCB (OBP-UCB), to maximize the cumulative expected reward without prior knowledge of the $\{\mu_k\}_{k=1}^{K}$. The OBP-UCB algorithm is formally described in Algorithm \ref{alg:single-obp} and uses UCB values to estimate arm rewards as in traditional MAB algorithms~\cite{UCB1_survey}. Define $\overline{\mu}_{i}(t)$ as the sample average of $\mu_{i}$ up to round $t$ and $n_{i}(t)$ as the number of times that arm $i$ has been observed. Define $\hat{\mu}_{i}(t) := \overline{\mu}_{i}(t) + \sqrt{\frac{2\log t}{n_{i}(t)}}$ as the UCB value of arm $i$ at round $t$.
At each round, the player ranks all the arms $i$ in descending order of $\hat{\mu}_{i}(t)$, and sets that order as $\boldsymbol{o}_t$. The player observes arms starting at $o^{(1)}_{t}$, stopping at the first good arm ($Y_{o_t^{(i)},t} = 1$) or when the list exhausts. She then updates the UCB values and enters the next round. Since we store and update each arm's UCB value, the storage and computing overhead grow only linearly with the number of arms $K$.

We can define and bound the \emph{regret} of this algorithm 
as the difference between the expected reward of the optimal policy (Lemma~\ref{lemma:descend_order}) and that of the real policy:
\begin{equation}\label{eq:R_single_1}
\begin{aligned}
    R(T) :=& \mathbb{E}[r^{*}(T)] - \mathbb{E}[r(T)]\\
    =& \sum_{t=1}^{T}\sum_{k=1}^{K}\Bigg\{(1 - k\,\tau)\mu_{k}\prod_{i=1}^{k-1}(1-\mu_{i}) - (1 - k\,\tau)\mu_{o^{(k)}_{t}}\prod_{i=1}^{k-1}(1-\mu_{o^{(i)}_{t}})\Bigg\}.
\end{aligned}
\end{equation}
\begin{theorem}\label{theorem:single}
The total expected regret can be bounded as: \\
$\mathbb{E}[R(T)] \leq \sum_{i=1}^{K-1}\Bigg\{i\,W_i\sum_{j=i+1}^{K} [\frac{8\log T}{\Delta_{i,j}} + (1 + \frac{\pi^2}{3})\Delta_{i,j}]\Bigg\}$, where $W_k := (1 - k\,\tau)\prod_{i=1}^{k-1}(1-\mu_{i})$ and $\Delta_{i,j} := \mu_i - \mu_j$.
\end{theorem}
The expected regret $\mathbb{E}[R(T)]$ is upper-bounded in the order of $O(K^2\log T)$, as also shown by~\cite{li2014almost}. However, our proof method is distinct from theirs and preserves the dependence on the arm rewards (through the $W_i$ in Theorem~\ref{theorem:single}). Since $W_k$ converges to 0 as $k\rightarrow\infty$, we expect that the constant in our $O(K^2\log T)$ bound will be small. Numerically, when there are more than 8 arms with expected rewards uniformly drawn from $(0, 1)$, our new regret bound is tighter than the result from~\cite{li2014almost} in 99\% of our experiments. Moreover, unlike the analysis in~\cite{li2014almost}, our regret analysis can be easily generalized to multi-player settings, as we show in the next section. 

Algorithms with better regret order in $T$ can be derived~\cite{combes2015learning}, but the regret bound of their proposed algorithm has a constant term (independent of $T$), $K^{2}\eta^2$, where $\eta = \prod_{i=1}^{K}(1-\mu_{i})^{-1}$. This constant term is exponential in $K$ so it can be significant if $K$ is large. The same work also provides a lower bound in the order of $\Omega(K\log T)$ when the player can only choose less than $K$ arms to pre-observe in each round.

\section{Centralized Multi-player Setting}\label{sec:centralized}
\captionsetup[figure]{labelfont=bf}
\begin{figure}[t]
\centering
\begin{subfigure}[t]{0.48\linewidth}
\centering
	\includegraphics[width=\textwidth]{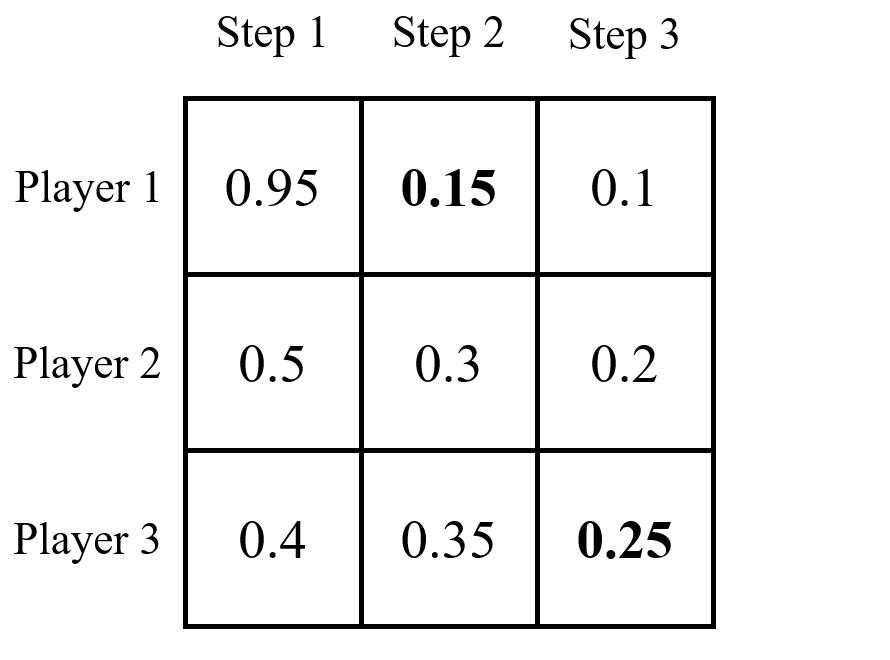}
	\caption{Non-greedy optimal policy.}
	\label{fig:opt_not_greedy}
\end{subfigure}
\begin{subfigure}[t]{0.38\linewidth}
    \centering
	\includegraphics[width=\textwidth]{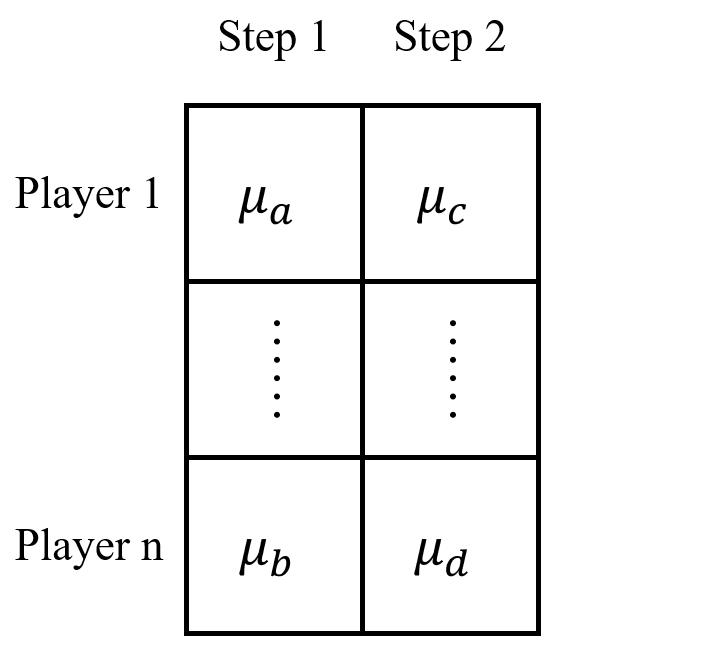}
	\caption{Assigning arms.}
	\label{fig:greedy_2}
\end{subfigure}
\caption{Multi-player observation lists, with rewards in the boxes.}
\end{figure}
In the multi-player setting, we still consider $K$ arms with i.i.d Bernoulli rewards; $Y_{k,t}$ denotes the realized reward of arm $k$ at round $t$, with an expected value $\mu_k \in [0,1]$. There are now $M \geq 1$ players $(M \leq K)$ making decisions on which arms to observe and play in each round. We define a {\bf collision} as two or more users playing the same arm in the same round, forcing them to share that arm's reward or even yielding zero reward for all colliding players, e.g., in CRNs. In this setting, simply running the OBP-UCB algorithm on all players will lead to severe collisions, since all users may tend to choose the same observation list and play the same arm. To prevent this from happening, we first consider the case where a central controller can allocate different arms to different players.

 At each round, the central controller decides pre-observation lists for all players; as in the single-player setting, each player sequentially observes the arms in its list and stops at the first good arm. The players report their observation results to the central controller, which uses them to choose future lists. A \emph{policy} consists of a set of pre-observation lists for all players.
Define $\boldsymbol{o}_{m,t} := (o^{(1)}_{m,t},o^{(2)}_{m,t},\dots,o^{(i)}_{m,t}, \dots)$ as the \textbf{pre-observation list} of player $m$ at round $t$, where $o^{(i)}_{m,t}$ represents the $i^{\text{th}}$ arm to be observed. 
The length of $\boldsymbol{o}_{m,t}$ can be less than $K$. Since collisions will always decrease the total reward, we only consider \emph{collision-free policies}, i.e., those in which players' pre-observation lists are disjoint. Policies that allow collisions are impractical in CRNs as they waste limited transmission energy and defeat the purpose of pre-observations (sensing channel availability), which allow users to find an available channel without colliding with primary users.
The expected overall reward of all players is then:
\begin{equation}\label{eq:reward_central}
    \mathbb{E}[r(T)] = \sum_{t=1}^{T}\sum_{m=1}^{M}\sum_{k=1}^{|\boldsymbol{o}_{m,t}|}\left\{(1 - k\tau)\mu_{o^{(k)}_{m,t}}\prod_{i=1}^{k-1}(1-\mu_{o^{(i)}_{m,t}})\right\}.
\end{equation}


Unlike in the single-player setting, the collision-free requirement now makes the expected reward for one player dependent on the decisions of other players.
Intuitively, we would expect that a policy of always using better arms in earlier steps would perform well.
We can in fact generalize Lemma \ref{lemma:descend_order} from the single-player setting:
\begin{lemma}\label{lemma:descend-multi}
Given a pre-observation list $\boldsymbol{o}_{m,t}$ for time $t$, player $m$ maximizes its expected reward at time $t$ by observing the arms in descending order of their rewards.
\end{lemma}

With Lemma~\ref{lemma:descend-multi}, we can consider the offline optimization of the centralized multi-player bandits problem. With the full information of expected rewards of all arms, i.e., $\{\mu_{i}\}^{K}_{i=1}$, the central controller allocates disjoint arm sets to different players, aiming to maximize the expected overall reward shown in \eqref{eq:reward_central}. We show in Theorem \ref{thm:multi-player-np-hard} that the offline problem is NP-hard.
\begin{theorem}\label{thm:multi-player-np-hard}
The offline problem of our centralized multi-player setting is NP-hard.
\end{theorem}
\begin{proof}
Define $x_{ij}=1$ if the central controller allocates arm $j$ to player $i$ and $0$ otherwise. The offline optimization problem can be formulated as:
\begin{align*}
&\text{max}
& & \sum_{i=1}^{M} \sum_{j=1}^{K}\Big\{\big[1 - (\sum_{k<j}x_{ik} + 1)\tau\big]x_{ij}\mu_{j}\prod_{k<j}(1-x_{ik}\mu_{k})\Big\} \notag\\
&\text{s.t.}
& & x_{ij} \in \{0,1\},\\
&&& \sum_{i=1}^{M} x_{ij} \le 1, \; j = 1, \ldots, K,
\end{align*}
where we define $\sum_{\emptyset} := 0$ and $\prod_{\emptyset} := 1$. 
We show the Weapon Target Assignment (WTA) problem~\cite{ahuja2007exact} with identical targets, which is NP-hard~\cite{proof_wta}, can be reduced in polynomial time to a special case of our problem with $\tau=0$: The WTA problem with identical targets aims to maximize the sum of expected damage done to all targets (mapped to be players), each of which can be targeted by possibly multiple weapons (mapped to be channels), where each weapon can only be assigned to at most one target and weapons of the same type have the same probability (mapped to be $\mu_k$) to successfully destroy any target. Then, it is equivalent to maximizing the expected reward of all players when $\tau=0$ in our problem.
\end{proof}

Although it is hard to find the exact offline optimal policy, Lemma~\ref{lemma:descend-multi} suggests that a {\textbf{collision-free greedy}} policy, which we also refer to as a {\emph{greedy policy}}, might be closed to the optimal one.
We first define the $i^{\text{th}}$ \textbf{observation step} in a policy as the set of arms in the $i^{\text{th}}$ positions of the players' observation lists, denoted by $\boldsymbol{s}_{i,t} := (o^{(i)}_{1,t},o^{(i)}_{2,t},\dots,o^{(i)}_{M,t})$ for each round $t$.
We define a {\emph{greedy policy}} as one in which at each observation step, the players greedily choose the arms with highest expected rewards from all arms not previously observed. Formally, assuming without loss of generality that $\mu_1 \geq \mu_2 \geq \dots \geq \mu_K$, in the $i$th observation step, players should observe different arms from the set $\boldsymbol{s}_{i,t} = \left\{(i - 1)M + 1, (i - 1)M + 2,\dots,iM\right\}$. 
In the simple \textbf{greedy-sorted policy}, for instance, player $m$ will choose arm $(i-1)M+m$ in the $i^{\text{th}}$ observation step. 
A potentially better candidate is the \textbf{greedy-reverse policy}: at each observation step, arms are allocated to players in the reverse order of the probability they observe an available arm from previous observation steps. Formally, in the $i$th observation step, arm $(i - 1)M + j$ is assigned to the player $m$ with the $j$th highest value of $\Pi_{l = 1}^{i - 1}(1 - \mu_{\boldsymbol{o}_{m,t}^{(l)}})$, or the probability player $m$ has yet not found an available arm. Experiments show that when there are 3 players and 9 arms with expected rewards uniformly drawn from $(0, 1)$, the greedy-reverse policy is the optimal greedy policy 90\% of the time. In fact,
%
\begin{lemma}\label{lemma:greedy_2}
When $K \leq 2M$, the optimal policy is the greedy-reverse policy.
\end{lemma}

In general, the optimal policy may not be the greedy-reverse one, or even a greedy policy. Figure~\ref{fig:opt_not_greedy} shows such a counter-intuitive example. In this example, player 1 should choose the arm with 0.15 expected reward, not the one with 0.25 expected reward, in step 2. Player 1 should reserve the higher-reward arm for player 3 in a later step, as player 3 has a lower chance of finding a good arm in steps 1 or 2. In practice, we expect these examples to be rare; they occur less than 30\% of the time in simulation.
Thus, we design an algorithm that allocates arms to players according to a specified greedy policy (e.g., greedy-sorted) and bound its regret.


\begin{algorithm}[tb]
 \caption{Centralized Multi-Player OBP (C-MP-OBP)}
 \begin{algorithmic}[1]
 \label{alg:c-mp-obp}
 \STATE \textbf{Initialization}: Pull all arms once and update $n_{i}(t)$, $\overline{\mu}_{i}(t)$, $\hat{\mu}_{i}(t)$ for all $i \in [K]$ 
 \WHILE{$t$}
    \STATE $\boldsymbol{\alpha} = \text{argsort}(\hat{\mu}_{1}(t), \hat{\mu}_{2}(t), \dots, \hat{\mu}_{K}(t))$;
    \FOR{$i = 1:L$}
        \STATE $\boldsymbol{s}_{i,t} = \boldsymbol{\alpha}[(i-1)*M+1\,:\,i*M]$
    \ENDFOR
    \FOR{$m = 1:M$}
        \FOR{$i = 1:L$}
            \STATE Observe arm $\boldsymbol{s}_{i,t}[m]$'s reward $Y_{\boldsymbol{s}_{i,t}[m],t}$;
            \STATE $n_{\boldsymbol{s}_{i,t}[m]}(t+1) = n_{\boldsymbol{s}_{i,t}[m]}(t) + 1$;
            \STATE $\overline{\mu}_{\boldsymbol{s}_{i,t}[m]}(t+1)$
            \STATE $= \left(\overline{\mu}_{\boldsymbol{s}_{i,t}[m]}(t) + Y_{\boldsymbol{s}_{i,t}[m],t}\right) / n_{\boldsymbol{s}_{i,t}[m]}(t+1)$;
            \IF{$Y_{\boldsymbol{s}_{i,t}[m],t} = 1$}
                \STATE Player $m$ plays arm $\boldsymbol{s}_{i,t}[m]$ for this round;
                \STATE $n_{\boldsymbol{s}_{j,t}[m]}(t+1) = n_{\boldsymbol{s}_{j,t}[m]}(t)$ for all $j > i$;
                \STATE $\overline{\mu}_{\boldsymbol{s}_{j,t}[m]}(t+1) = \overline{\mu}_{\boldsymbol{s}_{j,t}[m]}(t)$ for all $j > i$;
                \STATE break;
            \ENDIF
        \ENDFOR
    \ENDFOR
    \STATE Update $\hat{\mu}_{i}(t)$ for all $i \in [K]$;
    \STATE $t = t+1$;
 \ENDWHILE
 \end{algorithmic} 
\end{algorithm}

We propose an UCB-type online algorithm, \textbf{Centralized Multi-Player Observe-Before-Play} (C-MP-OBP), to learn a greedy policy without prior knowledge of the expected rewards $\{\mu_k\}_{k=1}^{K}$. The C-MP-OBP algorithm is described in Algorithm 1, generalizing the single-player setting. 
To simplify the discussion, we assume $K/M = L$, i.e., each player will have an observation list of the same length, $L$, when using a greedy policy. Note that if $K$ is not a multiple of $M$, we can introduce virtual arms with zero rewards to ensure $K/M = L$. At each round $t$, the central controller ranks all the arms in the descending order of $\hat{\mu}_{i}(t)$, the UCB value of arm $i$ at round $t$, and saves that order as $\boldsymbol{\alpha}$. Then it sets the first $M$ arms in $\boldsymbol{\alpha}$, $\boldsymbol{\alpha}[1:M]$, as $\boldsymbol{s}_{1,t}$, the second $M$ arms in $\boldsymbol{\alpha}$, $\boldsymbol{\alpha}[M+1:2M]$ as $\boldsymbol{s}_{2,t}$, and so on, assigning the arms in each list to players according to the specified greedy policy. Each player $m$'s observation list is then $\boldsymbol{o}_{m,t} = \left(\boldsymbol{s}_{1,t}[m],\ldots,\boldsymbol{s}_{L,t}[m]\right)$. 
At the end of this round, the central controller aggregates all players' observations to update the UCB values and enter the next round.

We define the \emph{regret}, $R(T) := \mathbb{E}[r^{*}(T)] - \mathbb{E}[r(T)]$, as the difference between the expected reward of the target policy and that of C-MP-OBP algorithm:
\begin{align}
    R(T) 
    = \sum_{t,m,k=1}^{T,M,L}
    \Bigg\{&(1 - k\tau)\mu_{(k-1)M+m}\prod_{i=1}^{k-1}(1-\mu_{(i-1)M+m}) \nonumber -(1 - k\tau)\mu_{o^{(k)}_{m,t}}\prod_{i=1}^{k-1}(1-\mu_{o^{(i)}_{m,t}})\Bigg\}.\label{eq:R_multi}\\
\end{align}
Defining $c_{\mu} := \frac{\mu_{\text{max}}}{\Delta_{\text{min}}}$, we show the following regret bound:
\begin{theorem}\label{theorem:multi}
The expected regret of C-MP-OBP is \\
$\mathbb{E}[R(T)]\leq c_{\mu} K^2 (L^2+L)\left(\frac{8\log T}{\Delta_{\min}} + (1 + \frac{\pi^2}{3})\Delta_{\max}\right)$, where $\Delta_{\max} = \underset{i< j}{\max}\,\mu_i-\mu_j$, $\Delta_{\min} = \underset{i< j}{\min}\,\mu_i-\mu_j$.
\end{theorem}
The expected regret $E[R(T)]$ is upper bounded in the order of $O(K^2L^2\log T)$
, compared to $O(K^2\log T)$ in the single-player setting. Thus, we incur a ``penalty'' of $L^2$ in the regret order, due to sub-optimal pre-observations' impact on the subsequent pre-observations of other users. We note that, if pre-observations are not allowed, we can adapt the proof of Theorem~\ref{theorem:multi} to match the lower bound of $O(KM\log T)$ given by \cite{besson2018multi}.

\section{Distributed Multi-player Setting}\label{sec:distributed}

We finally consider the scenario without a central controller or any means of communication between players. In the CRN setting, for instance, small Internet-of-Things devices may not be able to tolerate the overhead of communication with a central server. The centralized C-MP-OBP policy is then infeasible, 
and specifying a collision-free policy is difficult, as the players make their decisions independently. We propose a \textbf{Distributed Multi-Player Observe-Before-Play} (D-MP-OBP) online algorithm in which each player distributedly learns a ``good'' policy that effectively avoids collisions with others. Specifically, it converges to one of the offline collision-free greedy policies that we defined in Section \ref{sec:centralized}; we then show that D-MP-OBP can be adapted to achieve a pre-specified greedy policy, e.g., greedy-reverse. 
To facilitate the discussion, we define $\eta_k^{(t)}$ as an indicator that equals $1$ if more than one player plays arm $k$ in round $t$ and $0$ otherwise. As in the centralized setting, $o^{(k)}_{m,t}$ denotes the $k^{\text th}$ arm in player $m$'s observation list at round $t$. 

The D-MP-OBP algorithm is shown in Algorithm \ref{alg:d-mp-obp}. As in the C-MP-OBP algorithm, in each round, each player independently updates its estimate of the expected reward $(\mu_{k})$ for each arm $k$ using the UCB of $\mu_{k}$. Each player then sorts the estimated $\{\mu_{k}\}_{k=1}^K$ into descending order and groups the $K$ arms into $L$ sets. We still use $\boldsymbol{s}_{i,t}$ to denote the list of arms that the players observe in step $i$ at round $t$. Since users may have different lists $\boldsymbol{s}_{i,t}$ depending on their prior observations, we cannot simply allocate the arms in $\boldsymbol{s}_{i,t}$ to users. Instead, the users follow a randomized strategy in each step $i$ at round $t$. If there was a collision with another player on arm $i$ at round $t-1$ or the arm chosen in round $t-1$ does not belong to her own set $\boldsymbol{s}_{i,t}$, then the player uniformly at random chooses an arm from her $\boldsymbol{s}_{i,t}$ to observe. Otherwise, the player observes the same arm as she did in step $i$ in round $t-1$. If the arm is observed to be available, the player plays it and updates the immediate reward and the UCB of the arm. Otherwise, she continues to the next observation step. Note that this policy does not require any player communication. 

\begin{algorithm}[tb]
 \caption{Distributed Multi-Player OBP (D-MP-OBP)}
 \begin{algorithmic}[1]
 \label{alg:d-mp-obp}
 \STATE \textbf{Initialization}: Pull all arms once and update $n_{i}(t)$, $\overline{\mu}_{i}(t)$, $\hat{\mu}_{i}(t)$ for all $i \in [K]$ 
 \WHILE{$t$}
    \STATE $\boldsymbol{\alpha} = \text{argsort}(\hat{\mu}_{1}(t), \hat{\mu}_{2}(t), \dots, \hat{\mu}_{K}(t))$;
    \FOR{$i = 1:L$}
        \STATE $\boldsymbol{s}_{i,t} = \boldsymbol{\alpha}[(i-1)*M+1\,:\,i*M]$
    \ENDFOR
        \FOR{$i = 1:L$}
            \IF{$m^*_i = 0$ OR $m^*_i \notin \boldsymbol{s}_{i,t}$}
                \STATE The player uniformly at random selects an arm from $\boldsymbol{s}_{i,t}$ to observe and record the index of the chosen arm as $m^*_i$;
            \ENDIF
            \STATE Observe the reward $Y_{\boldsymbol{s}_{i,t}[m^*_i],t}$;
            \STATE $n_{\boldsymbol{s}_{i,t}[m^*_i]}(t+1) = n_{\boldsymbol{s}_{i,t}[m^*_i]}(t) + 1$;
            \STATE $\overline{\mu}_{\boldsymbol{s}_{i,t}[m^*_i]}(t+1)$
            \STATE $= \left(\overline{\mu}_{\boldsymbol{s}_{i,t}[m^*_i]}(t) + Y_{\boldsymbol{s}_{i,t}[m^*_i],t}\right) / n_{\boldsymbol{s}_{i,t}[m^*_i]}(t+1)$;
            \IF{$Y_{\boldsymbol{s}_{i,t}[m^*_i],t} = 1$}
                \STATE The player plays arm $\boldsymbol{s}_{i,t}[m*]$ for this round;
                \STATE $n_{\boldsymbol{s}_{j,t}[m*]}(t+1) = n_{\boldsymbol{s}_{j,t}[m*]}(t)$ for all $j > i$;
                \STATE $\overline{\mu}_{\boldsymbol{s}_{j,t}[m*]}(t+1) = \overline{\mu}_{\boldsymbol{s}_{j,t}[m*]}(t)$ for all $j > i$;
                \STATE break;
            \ENDIF
        \ENDFOR
        \IF{a collision occurs}
            \STATE Update $m^*_i = 0$;
        \ENDIF
    \STATE Update $\hat{\mu}_{i}(t)$ for all $i \in [K]$;
    \STATE $t = t+1$;
 \ENDWHILE
 \end{algorithmic} 
\end{algorithm}

To evaluate D-MP-OBP, we define a performance metric, $\text{Loss}(T)$, to be the maximum difference in total reward over $T$ rounds between any collision-free greedy policy and the reward achieved by D-MP-OBP.
%
Thus, unlike the regret $\mathbb{E}[R(T)]$ defined for our C-MP-OBP policy, $\mathbb{E}[\text{Loss}(T)]$ does not target a specific greedy policy. Moreover, unlike C-MP-OBP, our D-MP-OBP algorithm provides fairness in expectation for all players, as they have equal opportunities to use the best arms in each observation step.

\begin{theorem}\label{theorem:distributed}
The total expected loss, $\mathbb{E}[\text{Loss}(T)]$, of our distributed algorithm D-MP-OBP is logarithmic in $T$. 
\end{theorem}

We finally define the \textbf{D-MP-Adapt-OBP} algorithm, which adapts Algorithm \ref{alg:d-mp-obp} to steer the players towards a specific policy by adding a small extra term for each player. We define a function $f(\cdot)$ for each player to map the arm chosen in the first observation step to the arm chosen in the following steps given the predictions of each $\mu_k$. With some abuse of notation, we define $o^l_{m,t}$ as the arm chosen by player $m$ for step $l$ in round $t$. The function $f$ then steers the players to the collision-free greedy policy given by $o^{l+1}_{m, t} = f(o^l_{m,t}, \{\hat{\mu_k(t)}\}_{k=1}^K), \forall l=1, ..., L-1$ for each player $m$; we define the regret with respect to this policy. 

We can view the function $f$ as replacing the player index in the centralized setting with the relative ranking of the arm chosen by this player in prior observation steps. As an example, the greedy-sorted policy used in Section \ref{sec:centralized} is equivalent to: (1) letting players choose different arms, and (2) the player that chooses the arm in position $m$ continuing to choose the arm with the $m^{\text th}$ best reward of its set $\boldsymbol{s}_{i,t}$ in each subsequent step. 
Thus, we can steer the players to specific observation lists within a given collision-free greedy policy. Their decisions then converge to the specified policy. 
\begin{theorem}\label{theorem:distributed_adapt}
The expected regret, $\mathbb{E}[R(T)]$ of our distributed algorithm D-MP-Adapt-OBP is logarithmic in $T$.
\end{theorem}
We observe from the proof of Theorem \ref{theorem:distributed_adapt} that the regret is combinatorial in $M$ but logarithmic in $T$, unlike the centralized multi-player setting's $O(K^2 L^2\log T)$ regret in Theorem~\ref{theorem:multi}. This scaling with $M$ comes from the lack of coordination between players and the resulting collisions.

\section{Experiments}\label{sec:exp}
We validate the theoretical results from Sections~\ref{sec:single}--\ref{sec:distributed} with numerical simulations. We summarize our results as follows:


{\bf Sublinear regret:} We show in Figure~\ref{fig:regret} that our algorithms in the single-player, multi-player centralized, and multi-player distributed settings all achieve a sublinear regret, respectively defined relative to the single-player offline optimal (Lemma~\ref{lemma:descend_order}), the greedy-sorted policy, and a collision-free-greedy-random policy that in each step greedily chooses the set of arms but randomly picks one collision-free allocation. 
%
Figure~\ref{fig:multi_regret} shows our C-MP-OBP algorithm's regret is even negative for a few runs: by deviating from the greedy-sorted policy towards the true optimum, the C-MP-OBP algorithm may obtain a higher reward. 
The regret of D-MP-OBP in Figure~\ref{fig:dis_regret} is larger than that of C-MP-OBP, likely due to collisions in the distributed setting.
\captionsetup[figure]{labelfont=bf}
\begin{figure}
    \begin{subfigure}[b]{0.3\linewidth}
    \centering
        \includegraphics[width=\textwidth]{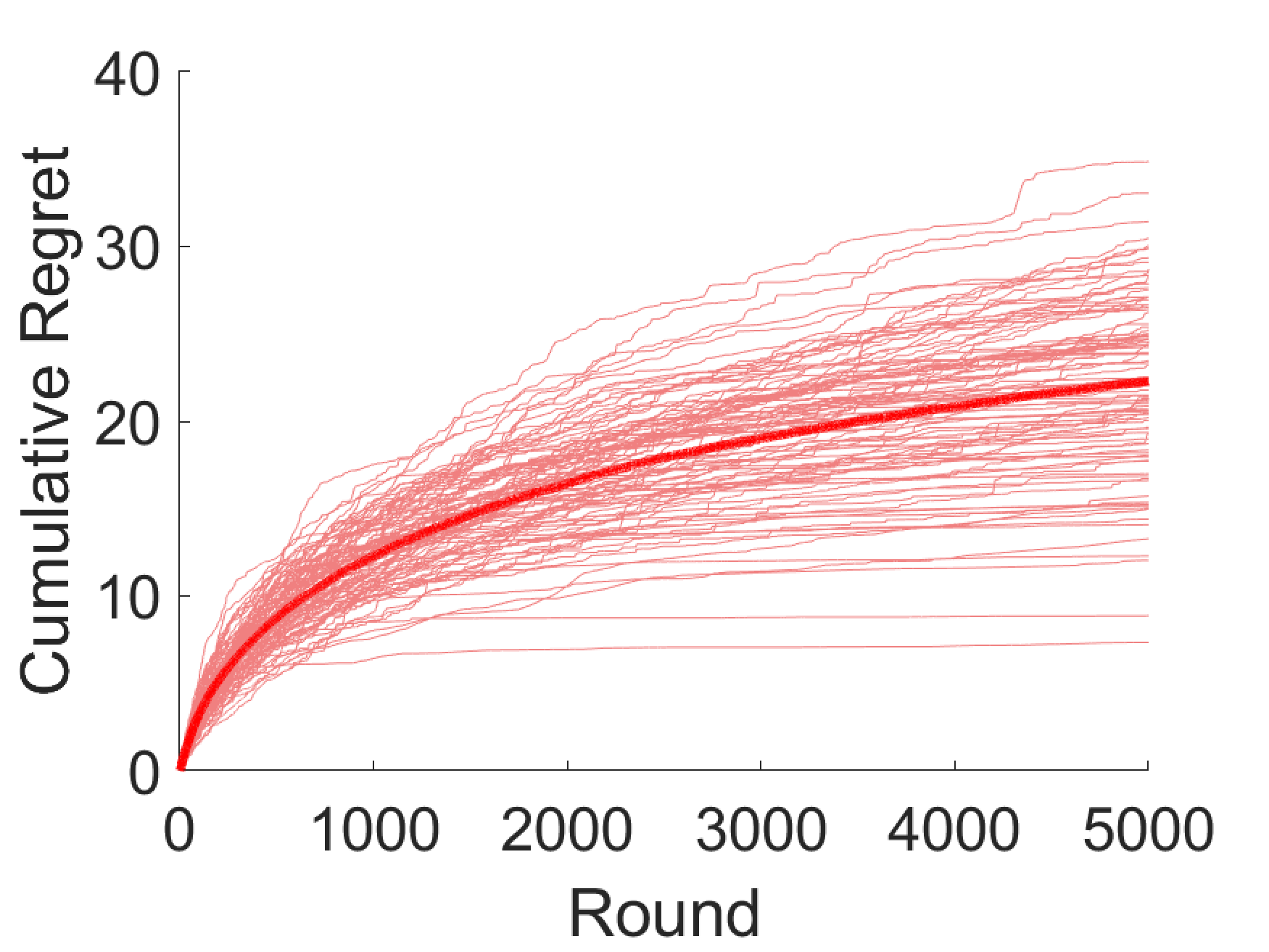}
        \caption{OBP-UCB.}
        \label{fig:single_regret}
    \end{subfigure}
    \hspace{2mm}
    \begin{subfigure}[b]{0.3\linewidth}
    \centering
        \includegraphics[width=\textwidth]{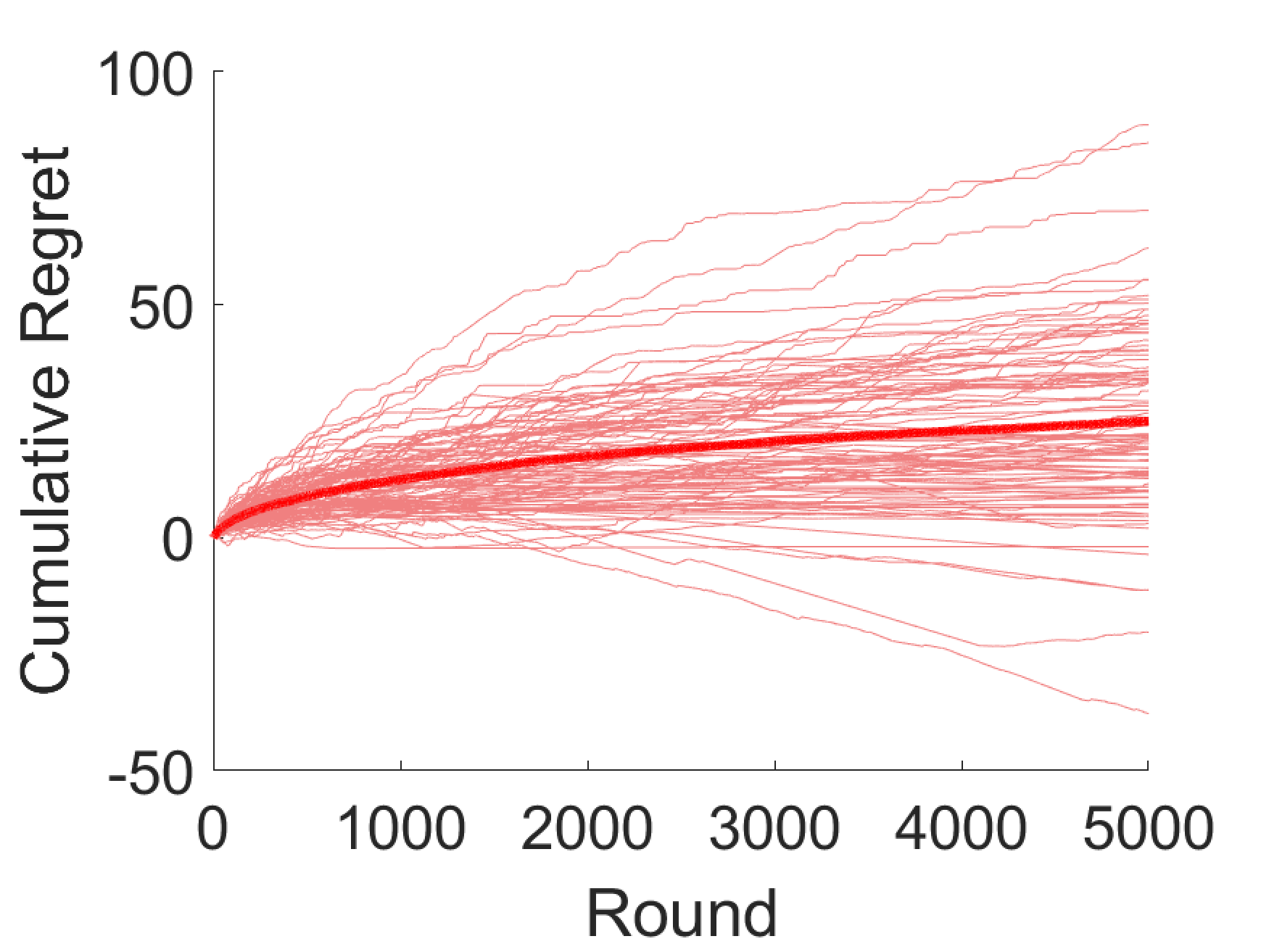}
        \caption{C-MP-OBP.}
        \label{fig:multi_regret}
    \end{subfigure}
    \hspace{2mm}
    \begin{subfigure}[b]{0.3\linewidth}
    \centering
        \includegraphics[width=\textwidth]{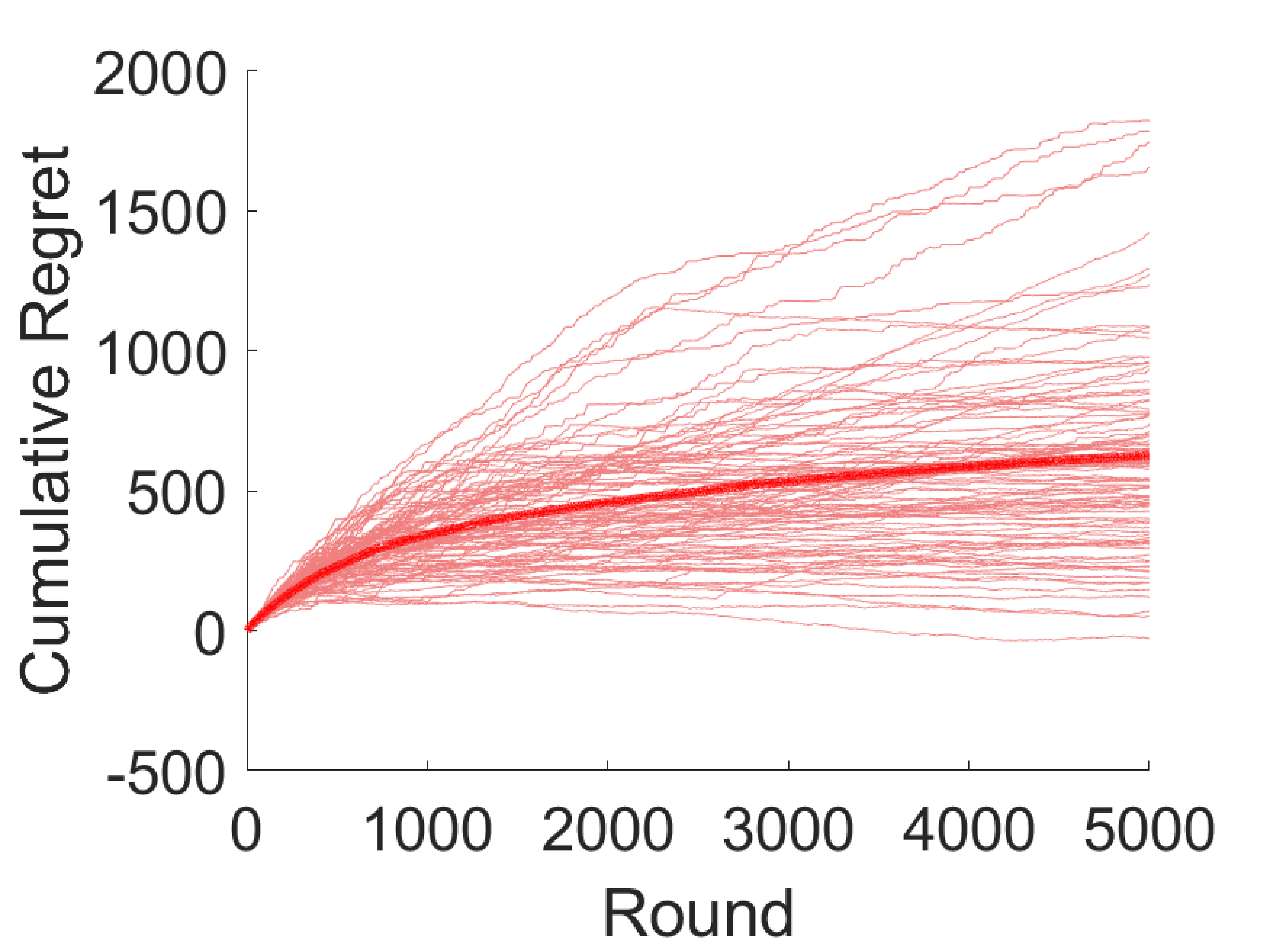}
        \caption{D-MP-OBP.}
        \label{fig:dis_regret}
    \end{subfigure}
\caption{Sublinear regret in each setting. Each line represents an experiment run with randomly chosen reward distributions; the bold line is the average over 100 runs.}
\label{fig:regret}
\end{figure}

\begin{table}[t]
\centering
\begin{tabular}{lllll}
\hline
$\tau$ & single-opt & random & single-real & random-real \\
\hline
0.01   & 102\%              & 5\%    & 76\%           & 6\%               \\
0.05   & 92\%               & 34\%   & 71\%           & 47\%              \\
0.1    & 78\%               & 140\%  & 63\%           & 245\%             \\
\hline
\end{tabular}
\caption{Average \% reward improvements of OBP-UCB}
\label{tab:1}
\end{table}
\begin{table}[t]
\centering
\begin{tabular}{lllll}
\hline
$\tau$ & single-opt & random & single-real & random-real \\
\hline
0.1                         & 41\%, 27\%               & 7\%, 39\%    & 35\%, 198\%           & 4\%, 30\%               \\
0.2                         & 33\%, 20\%               & 15\%, 47\%   & 28\%, 183\%           & 10\%, 36\%              \\
0.3                         & 22\%, 11\%               & 30\%, 60\%   & 19\%, 165\%           & 20\%, 47\%              \\
\hline
\end{tabular}
\caption{Average C-MP-OBP, D-MP-OBP \% improvement.}
\label{tab:2}
\end{table}
{\bf Superiority to baseline strategies:} We show in Tables \ref{tab:1} and \ref{tab:2} that our algorithms consistently outperform two baselines, in both synthetic reward data ($K = 9$ arms with expected rewards uniformly drawn from $[0,0.5]$ and $M = 3$ players for multi-player settings) and real channel availability traces~\cite{real_trace}. Our first baseline is a \textbf{random heuristic} (called \textbf{random} for synthetic data and \textbf{random-real} for real data trace) in which users pre-observe arms uniformly at random and play the first available arm. Comparisons to this baseline demonstrate the value of strategically choosing the order of the pre-observations. Our second baseline is an \textbf{optimal offline single-observation policy} (\textbf{single-opt}), which allocates the arms with the $M$ highest rewards to each player (in the single-player setting, $M = 1$). 
These optimal offline policies are superior to any learning-based policy with a single observation, so comparisons with this baseline demonstrate the value of pre-observations. When the rewards are drawn from a real data trace, they may no longer be i.i.d. Bernoulli distributed, so these offline policies are no longer truly ``optimal.'' Instead, we take a \textbf{single-observation UCB algorithm} (\textbf{single-real}) as the baseline; this algorithm allocates the arms with the top $M$ ($\geq 1$) highest UCB values to different users, and each player still observes and plays one such arm in each round.

Tables \ref{tab:1} and \ref{tab:2} show the average  improvements in the cumulative reward achieved by our algorithms over the baselines after 5000 rounds over 100 experiment repetitions with different $\tau$. In each setting, increasing $\tau$ causes the improvement over the random baseline to increase: when $\tau$ is small, there is little cost to mis-ordered observations, so the random algorithm performs relatively well. Conversely, increasing $\tau$ narrows the reward gap with the single-observation baseline: as pre-observations become more expensive, allowing users to make them does not increase the reward as much.

{\bf Effect of $\mu$:} We would intuitively expect that increasing the average rewards $\mu_i$ would increase the reward gap with the random baseline: it is then more important to pre-observe ``good'' arms first, to avoid the extra costs from pre-observing occupied arms. We confirm this intuition in each of our three settings. However, increasing the $\mu$'s does not always increase the reward gap with the single-observation baseline, since if the $\mu$'s are very low or very high, pre-observations are less valuable. When the $\mu$'s are small, the player would need to pre-observe several arms to find an available one, decreasing the final reward due to the cost of these pre-observations. When the $\mu$'s are large, simply choosing the best arm is likely to yield a high reward, and the pre-observations would add little value. Figures~\ref{fig:single_mu_offline} and \ref{fig:single_mu_random} plot the reward gap with respect to $x$ ( $\mu$'s are drawn from $U(0, x)$) : an increase in $x$ increases the reward gap with the random baseline, but has a non-monotonic effect compared to the single-observation baseline. Similar trends in multi-player settings are shown in the appendix.

\captionsetup[figure]{labelfont=bf}
\begin{figure}
    \begin{subfigure}[b]{0.48\linewidth}
    \centering
        \includegraphics[width=\textwidth]{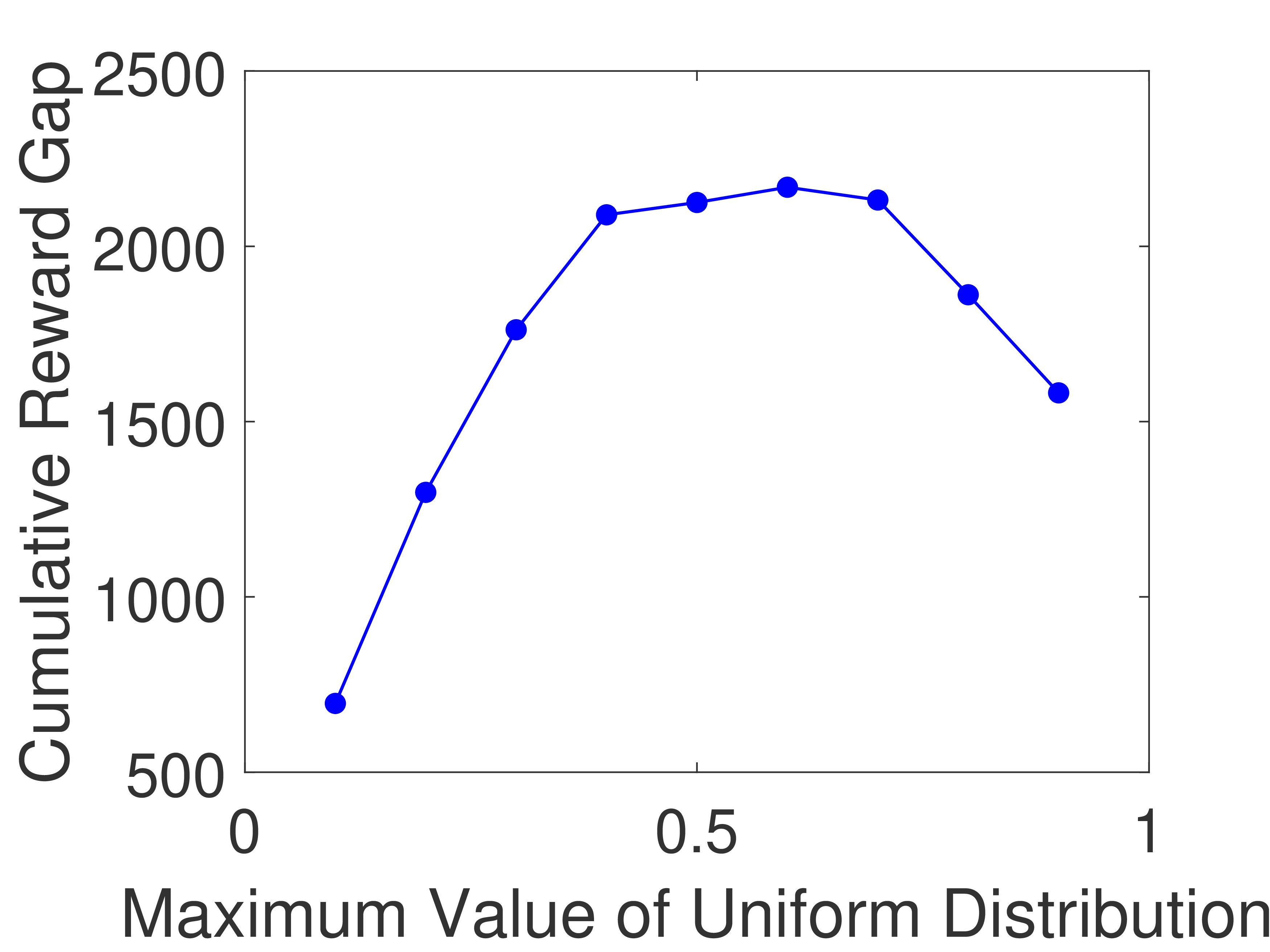}
        \caption{Single-observation baseline.}
        \label{fig:single_mu_offline}
    \end{subfigure}
    \hspace{2mm}
    \begin{subfigure}[b]{0.48\linewidth}
    \centering
        \includegraphics[width=\textwidth]{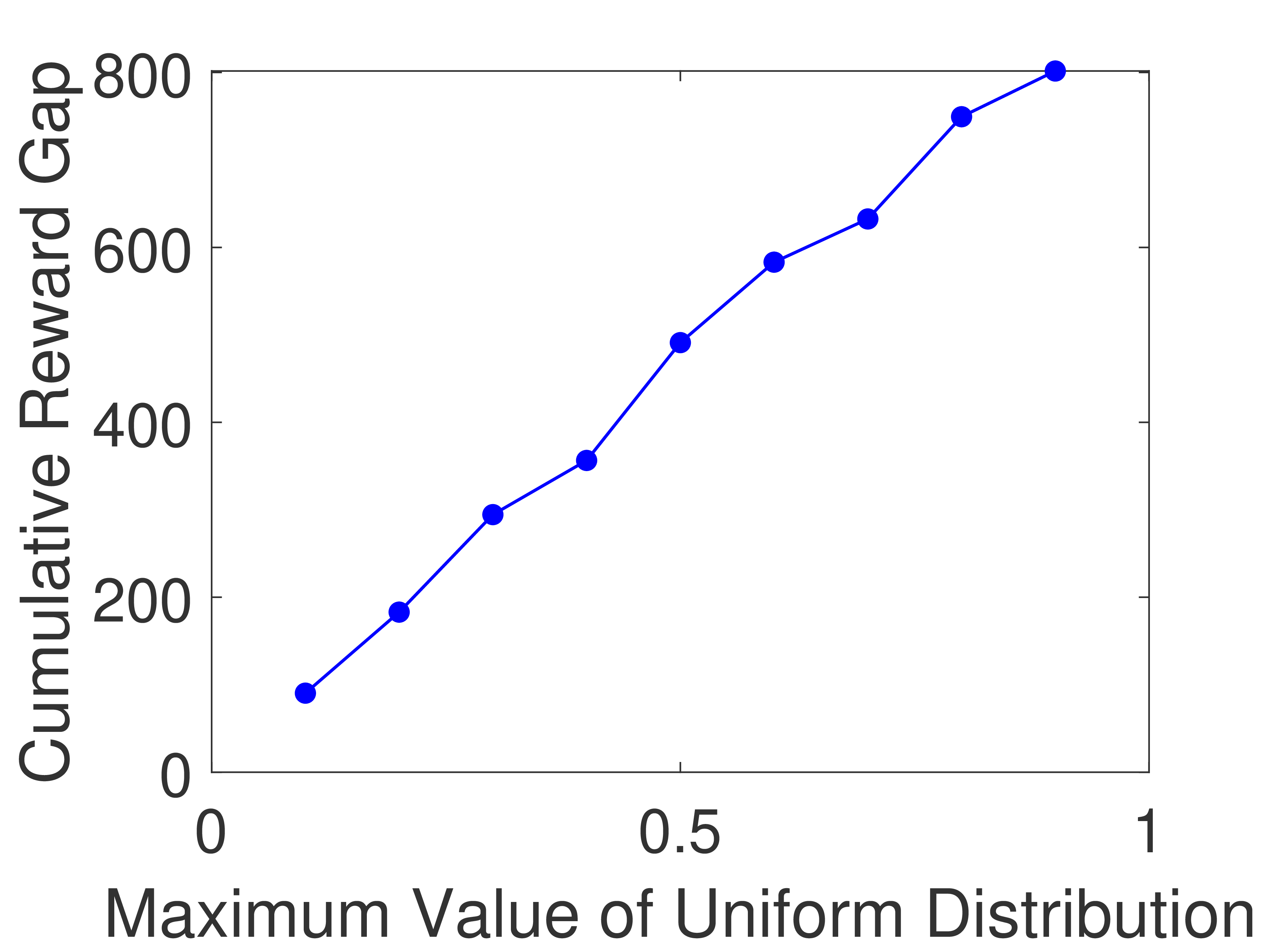}
        \caption{Random baseline.}
        \label{fig:single_mu_random}
    \end{subfigure}
\caption{Average cumulative reward gaps in the single-player (OBP-UCB) setting after 5000 rounds over 100 experiments, when $\tau = 0.1$ and $K = 9$ arms with expected rewards $\mu$'s uniformly drawn from the range $[0, x]$.}
\label{fig:mu}
\end{figure}
\section{Discussion and Conclusion}\label{sec:discuss}
In this work, we introduce {\bf pre-observations} into multi-armed bandit problems. 
Such pre-observations introduce new technical challenges to the MAB framework, as players must not only learn the best set of arms, but also the optimal order in which to pre-observe these arms. This challenge is particularly difficult in multi-player settings, as each player must learn an observation set of arms that avoids collisions with other players. We develop algorithms for both the single- and multi-player settings and show that they achieve logarithmic regret over multiple rounds. As one of the first works to consider pre-observations, however, we leave several problems open for future work. One might, for instance, consider user arrivals and departures, which would affect the offline optimal observation lists; or temporal reward correlations. Both of these would likely arise in our motivating scenario of cognitive radio networks, as devices move in and out of range and channel incumbents exhibit temporal behavior patterns. Another challenging extension would be to consider cases with more limited collisions, where one arm might serve multiple users (e.g., if an ``arm'' is a city block when users are searching for parking spaces). In such cases, we must learn not just the probability that the arm is available (i.e., its expected reward) but also the full distribution of the number of users that the arm can accommodate.

\bibliographystyle{unsrt}
\bibliography{main}

\newpage
\appendix
{\bf \Large Appendix}
\section{Proof of Lemma \ref{lemma:descend_order} and \ref{lemma:descend-multi}} \label{app:descend_order}
\begin{proof}
Assume there exists an observation list $\boldsymbol{o}_\text{old}$ such that ${o}_\text{old}^{(i)} = b, {o}_\text{old}^{(j)} = a$, and $i < j, \mu_a > \mu_b$. In other words, the $i^{\text{th}}$ arm to be observed in $\boldsymbol{o}_\text{old}$ has less expected reward than the $j^{\text{th}}$ arm. Now let us consider a new observation list $\boldsymbol{o}_\text{new}$, which switches arms $a$ and $b$ in $\boldsymbol{o}_\text{old}$ and leaves the other arms unchanged. Define the one-round expected reward of $\boldsymbol{o}_\text{old}$ and $\boldsymbol{o}_\text{new}$ as $r_\text{old}$ and $r_\text{new}$. From \eqref{eq:r(T)_single}, we can find that the gap between $r_\text{old}$ and $r_\text{new}$ is only caused by the $i^{\text{th}}$ to the $j^{\text{th}}$ arm in the observation list, so we get:
\begin{equation}
\begin{aligned}
    r_\text{new} - r_\text{old} &= \sum_{k=i}^{j}\Bigg\{(1 - k\,\tau)\mu_{o^{(k)}_{\text{new}}}\prod_{x=1}^{k-1}(1-\mu_{o^{(x)}_{\text{new}}})-(1 - k\,\tau)\mu_{o^{(k)}_{\text{old}}}\prod_{x=1}^{k-1}(1-\mu_{o^{(x)}_{\text{old}}}) \Bigg\}\\
    & = \prod_{x=1}^{i-1}(1-\mu_{o^{(x)}_{\text{new}}}) \Bigg \{(1 - i\,\tau)(\mu_a - \mu_b)-\sum_{k=i+1}^{j-1} \Big \{ (1 - k\,\tau)(\mu_a - \mu_b) \mu_{o^{(k)}_{\text{new}}}\prod_{x=i+1}^{k-1}(1-\mu_{o^{(x)}_{\text{new}}})\Big \} -  \\
    & (1 - j\,\tau)(\mu_a - \mu_b) \prod_{x=i+1}^{j-1}(1-\mu_{o^{(x)}_{\text{new}}})\Bigg \}\\
    & > \prod_{x=1}^{i-1}(1-\mu_{o^{(x)}_{\text{new}}}) \Bigg \{(1 - i\,\tau)(\mu_a - \mu_b) -  (1 -(i+1)\,\tau) (\mu_a - \mu_b) \sum_{k=i+1}^{j}  \mu_{o^{(k)}_{\text{new}}}\prod_{x=i+1}^{k-1}(1-\mu_{o^{(x)}_{\text{new}}})-\\
    &(1 - (i+1)\,\tau)(\mu_a - \mu_b) \prod_{x=i+1}^{j-1}(1-\mu_{o^{(x)}_{\text{new}}})\Bigg \} \\
    & \geq \prod_{x=1}^{i-1}(1-\mu_{o^{(x)}_{\text{new}}}) \Bigg \{(1 - i\,\tau)(\mu_a - \mu_b) - (1 -(i+1)\,\tau) (\mu_a - \mu_b) \Bigg \} \\
    & \geq \prod_{x=1}^{i-1}(1-\mu_{o^{(x)}_{\text{new}}}) \Bigg \{\tau\,(\mu_a - \mu_b) \Bigg \}\\
    & \geq 0.
\end{aligned}
\end{equation}
Thus, the expected reward of $\boldsymbol{o}_\text{new}$ is always larger than that of $\boldsymbol{o}_\text{old}$. As a result, exchanging arms $a$ and $b$ in $\boldsymbol{o}_\text{old}$ always improves the expected reward. We can then conclude that
the optimal policy for the single-player setting is $\boldsymbol{o}_t^{*}=(1, 2, \dots, K)$, which is Lemma \ref{lemma:descend_order}. For the centralized multi-player setting, similarly, the optimal ordering is where no arm has lower expected reward than any arm observed after it, which concludes the proof of Lemma \ref{lemma:descend-multi}. 
\end{proof}

\section{Proof of Theorem~\ref{theorem:single}} \label{app:single}
\begin{proof}
To prove Theorem \ref{theorem:single}, let us firstly rewrite \eqref{eq:R_single_1} as:
\begin{equation}\label{eq:R_single}
\begin{aligned}
    R(T) & = \sum_{t=1}^{T}\sum_{k=1}^{K}\Bigg\{(1 - k\,\tau)\mu_{k}\prod_{i=1}^{k-1}(1-\mu_{i}) - (1 - k\,\tau)\mu_{o^{(k)}_{t}}\prod_{i=1}^{k-1}(1-\mu_{o^{(i)}_{t}})\Bigg\}\\
      & \leq \sum_{t=1}^{T}\sum_{k=1}^{K}\Bigg\{(1 - k\,\tau)(\mu_{k}-\mu_{o^{(k)}_{t}})\prod_{i=1}^{k-1}(1-\mu_{i})\Bigg\}.
\end{aligned}
\end{equation}
The last inequality holds since $\prod_{i=1}^{k-1}(1-\mu_{i})$ is always not greater than $\prod_{i=1}^{k-1}(1-\mu_{o^{(i)}_{t}})$ for any $\boldsymbol{o}_t$ when $\mu_1 \geq \mu_2 \geq \dots \geq \mu_K$. Now let us focus on this inequality. At round $t$, if $o^{(k)}_{t} > k$ (i.e., the $k$th pre-observed arm has better average reward than arm $k$), then $\mu_{k}-\mu_{o^{(k)}_{t}} \geq 0$ and the regret for $o^{(k)}_{t}$ is nonnegative; if $o^{(k)}_{t} < k$, then $\mu_{k}-\mu_{o^{(k)}_{t}} \leq 0$, and the regret for $o^{(k)}_{t}$ is non-positive. In order to upper bound $R(T)$, we can ignore the negative terms and only count the positive regrets for all $o^{(k)}_{t} > k$. These positive regrets come from observing arms with lower expected rewards before those with higher expected rewards. Letting $W_k := (1 - k\,\tau)\prod_{i=1}^{k-1}(1-\mu_{i})$ and $\Delta_{i,j} := \mu_i - \mu_j$, the total regret can be bounded as:
\begin{equation}\label{eq:R(T)}
\begin{aligned}
    R(T) \leq & \sum_{t=1}^{T}\sum_{k=1}^{K}\Bigg\{W_k\,\Delta_{k,o^{(k)}_{t}}\mathbbm{1}\{o^{(k)}_{t} > k\}\Bigg\}.
\end{aligned}
\end{equation}
Define $T_{i,j}$ as the number of times that the $i^{\text{th}}$ arm to be observed in $\boldsymbol{o}_t$ is arm $j$, i.e., $T_{i,j} := \sum_{t=1}^{T}\mathbbm{1}\{o^{(i)}_{t} = j\}$. We then rewrite \eqref{eq:R(T)}:
\begin{equation}\label{eq:R(T)_single_3}
\begin{aligned}
    R(T) &\leq \sum_{t=1}^{T}\sum_{i=1}^{K-1}\sum_{j=i+1}^{K}\Bigg\{W_i\,\Delta_{i,j}\mathbbm{1}\{o^{(i)}_{t} = j\}\Bigg\}\\
    &= \sum_{i=1}^{K-1}\sum_{j=i+1}^{K}\Bigg\{W_i\,\Delta_{i,j}\sum_{t=1}^{T}\mathbbm{1}\{o^{(i)}_{t} = j\}\Bigg\}\\
    &= \sum_{i=1}^{K-1}\sum_{j=i+1}^{K}\Bigg\{W_i\,\Delta_{i,j}\,T_{i,j}\Bigg\}.
\end{aligned}
\end{equation}
In order to bound $\mathbb{E}[R(T)]$, we need to bound $\mathbb{E}[T_{i,j}]$ for all $i<j$.
\begin{lemma}
\label{lemma:Tij}
$\forall\,  i, j \in [K]$ with $i < j$, under Algorithm 1, $\mathbb{E}[T_{i,j}] \leq i\,(\frac{8\log T}{\Delta_{i,j}^2} + 1 + \frac{\pi^2}{3})$.
\end{lemma}
\begin{proof}
Algorithm 1 sorts the UCB values to determine the pre-observation list $\boldsymbol{o}_t$, so $T_{i,j}$ is equal to the number of times that $\hat{\mu}_j(t)$, the UCB value of arm $j$, is the $i^{\text{th}}$ largest one in $\hat{\boldsymbol{\mu}}(t)$. In that case, at least one arm in the set $\{1, 2, \dots, i\}$ has smaller UCB value than $\hat{\mu}_j(t)$, since at most $i-1$ arms have larger UCB values than $\hat{\mu}_j(t)$. Thus, $T_{i,j}$ can be bounded by the number of times that the minimum UCB value of arms $\{1, 2, \dots, i\}$ is less than $\hat{\mu}_j(t)$:
\begin{equation}
\begin{aligned}
    T_{i,j} &\leq \sum_{t=1}^{T} \mathbbm{1}\{\min_{k\in [i]} \hat{\mu}_k(t) \leq \hat{\mu}_j(t)\}\\
    &\leq \sum_{t=1}^{T} \sum_{k=1}^{i} \mathbbm{1}\{\hat{\mu}_k(t) \leq \hat{\mu}_j(t)\}\\
    &\leq \sum_{k=1}^{i}\sum_{t=1}^{T} \mathbbm{1}\{\hat{\mu}_k(t) \leq \hat{\mu}_j(t)\}.
\end{aligned}
\end{equation}
Since $i < j$ and $k \in [i]$, we can bound $\sum_{t=1}^{T} \mathbbm{1}\{\hat{\mu}_k(t) \leq \hat{\mu}_j(t)\}$
using the same idea to bound the number of times of choosing sub-optimal arms in traditional UCB1 algorithm \cite{UCB1_survey}. We can get:
\begin{equation}
\begin{aligned}
    \mathbb{E}[T_{i,j}] &\leq \sum_{k=1}^{i}\mathbb{E}\left[\sum_{t=1}^{T} \mathbbm{1}\{\hat{\mu}_k(t) \leq \hat{\mu}_j(t)\}\right]\\
    &\leq \sum_{k=1}^{i}\left\{ \frac{8\log T}{\Delta_{k,j}^2} + 1 + \frac{\pi^2}{3}\right\}\\
    &\leq i\,(\frac{8\log T}{\Delta_{i,j}^2} + 1 + \frac{\pi^2}{3}),
\end{aligned}
\end{equation}
which concludes the proof.
\end{proof}
Combining Lemma \ref{lemma:Tij} and \eqref{eq:R(T)_single_3} gives the upper bound of the expected regret in Theorem \ref{theorem:single}:
\begin{equation}
\begin{aligned}
     \mathbb{E}[R(T)] &\leq \sum_{i=1}^{K-1}\sum_{j=i+1}^{K}\Bigg\{W_i\,\Delta_{i,j}\mathbb{E}[T_{i,j}]\Bigg\}\\
      &\leq \sum_{i=1}^{K-1}\Bigg\{i\,W_i\sum_{j=i+1}^{K} [\frac{8\log T}{\Delta_{i,j}} + (1 + \frac{\pi^2}{3})\Delta_{i,j}]\Bigg\}.
\end{aligned}
\end{equation}
\end{proof}

\section{Proof of Lemma \ref{lemma:greedy_2}}\label{app:greedy_2}
\begin{proof}
When $K \leq 2M$, there are at most two observation steps for each player. As shown in Figure~\ref{fig:greedy_2}, we assume $\mu_a,\,\mu_b$ is larger than $\mu_c,\,\mu_d$, and now the expected reward for player 1 and player 2 is $r_{\text{old}} = (1-\tau)(\mu_a+\mu_b)+(1-2\tau)[(1-\mu_a)\mu_c +(1-\mu_b)\mu_d]$. If we switch arms with $\mu_a$ and $\mu_d$, the expected reward becomes $r_{\text{new}} = (1-\tau)(\mu_d+\mu_b) + (1-2\tau)[(1-\mu_d)\mu_c + (1-\mu_b)\mu_a]$, so the gap between them is:
\begin{equation}
\begin{aligned}
    r_{\text{old}} - r_{\text{new}} &= (1-\tau)(\mu_a - \mu_d) - (1-2\tau)(1-\mu_b+\mu_c)(\mu_a - \mu_d)\\
    &\geq  (1-\tau)(\mu_a - \mu_d) - (1-2\tau)(\mu_a - \mu_d)\\
    &\geq  \tau(\mu_a - \mu_d)\\
    &\geq  0.
\end{aligned}
\end{equation}
So the expected reward will only decrease when switching an arm with lower expected reward from step 2 to step 1, which ensures the optimal offline policy to be a greedy policy. 
\end{proof}

\section{Proof of Theorem \ref{theorem:multi}}\label{app:theorem:multi}
\begin{proof}
Unlike \eqref{eq:R_single}, we cannot directly upper bound \eqref{eq:R_multi} since $\prod_{i=1}^{k-1}(1-\mu_{(k-1)M+m})$ is not always less than $\prod_{i=1}^{k-1}(1-\mu_{o^{(i)}_{m,t}})$. Due to the correlation between different players' expected rewards, the analysis of the regret is challenging. Our idea is to decompose the regret into two parts: the first part is the regret caused by putting the arms into the wrong observation steps; the second part is the regret caused by different arm allocations within one observation step, where the set of arms to be allocated is correct. Define $R^{s}_{i,k}(T)$ as the regret caused by putting arm $i>kM$ into a wrong observation step $k$, when all previous observation steps are correct. In Figure~\ref{fig:greedy_2}'s illustration, this corresponds to an arm being placed in the incorrect column, though the arms in prior columns are placed correctly. We will show why this is sufficient to capture the first part of the total regret. Define $R^{a}_{i,k}$ as the regret caused by arm $i$ in the correct observation step $k$, i.e., $(k-1)M+1 \leq i \leq kM$, to capture the second part of the total regret. This regret corresponds to arm $i$ being placed in the correct column $k$ but incorrect row in Figure~\ref{fig:greedy_2}. We can then rewrite the total regret as:
\begin{equation}
\label{equ:Rs+Ra}
    R(T) \leq \sum_{k=1}^{L} \left\{\sum_{i>kM}^{K}R^{s}_{i,k}(T)+\sum_{i=(k-1)M+1}^{kM} R^{a}_{i,k}(T)\right\}.
\end{equation}
In order to find the upper bound of $R(T)$, we need to bound $R^{s}_{i,k}(T)$ and $R^{a}_{i,k}(T)$ separately. Let us first consider $R^{s}_{i,k}(T)$. Denote $T^{s}_{i,k}$ as the number of times that arm $i$ is in the $k^{\text{th}}$ observation step. Under algorithm \ref{alg:c-mp-obp}, we can bound $\mathbb{E}[T^{s}_{i,k}]$ for all $i > kM$.
\begin{lemma}
\label{lemma:Ts}
We have $\mathbb{E}[T^{s}_{i,k}] \leq kM\,(\frac{8\log T}{\Delta_{kM,i}^2} + 1 + \frac{\pi^2}{3}), \forall\ i > kM$.
\end{lemma}
\begin{proof}
Algorithm 2 sorts the UCB values to determine $\boldsymbol{o}_{m,t}$, so $T^{s}_{i,k}$ is equal to the number of times that $\hat{\mu}_i(t)$, the UCB value of arm $i$, should be at least the $kM^{\text{th}}$ largest one in $\hat{\boldsymbol{\mu}}(t)$. In that case, at least one arm in the set $\{1, 2, \dots, kM\}$ has smaller UCB value than $\hat{\mu}_i(t)$, since at most $kM-1$ arms have larger UCB values than $\hat{\mu}_i(t)$. Thus, $T^{s}_{i,k}$ can be bounded by the number of times that the minimum UCB value of arms $\{1, 2, \dots, kM\}$ is less than $\hat{\mu}_i(t)$:
\begin{equation}
\begin{aligned}
    T^{s}_{i,k} &\leq \sum_{t=1}^{T} \mathbbm{1}\{\min_{j\in [kM]} \hat{\mu}_j(t) \leq \hat{\mu}_i(t)\}\\
    &\leq \sum_{t=1}^{T} \sum_{j=1}^{kM} \mathbbm{1}\{\hat{\mu}_j(t) \leq \hat{\mu}_i(t)\}\\
    &\leq \sum_{j=1}^{kM}\sum_{t=1}^{T} \mathbbm{1}\{\hat{\mu}_j(t) \leq \hat{\mu}_i(t)\}.
\end{aligned}
\end{equation}
Since $i > kM$ and $j \in [kM]$, we can bound $\sum_{t=1}^{T} \mathbbm{1}\{\hat{\mu}_j(t) \leq \hat{\mu}_i(t)\}$
using the same idea to bound the number of times of choosing sub-optimal arms in traditional UCB1 algorithm \cite{UCB1_survey}. We can get:
\begin{equation}
\begin{aligned}\label{eq:number-wrong-steps}
    \mathbb{E}[T^{s}_{i,k}] &\leq \sum_{j=1}^{kM}\mathbb{E}\left[\sum_{t=1}^{T} \mathbbm{1}\{\hat{\mu}_j(t) \leq \hat{\mu}_i(t)\}\right]\\
    &\leq \sum_{j=1}^{kM}\left\{ \frac{8\log T}{\Delta_{j,i}^2} + 1 + \frac{\pi^2}{3}\right\}\\
    &\leq kM\,(\frac{8\log T}{\Delta_{kM,i}^2} + 1 + \frac{\pi^2}{3}),
\end{aligned}
\end{equation}
which concludes the proof.
\end{proof}
In order to find the upper bound of $R^{s}_{i,k}(T)$, we also need to consider the value of regret in each round. Define $R^{\text{max}}_{k}$ as the maximum one-round regret for one player when he has a wrong arm in the $k^{\text{th}}$ observation step and all previous selected arms are correct. We consider the worst case to get this maximum regret, which puts this wrong arm $i$ on the first place in the $k^{\text{th}}$ observation step, i.e.,  $o^{(k)}_{1,t} = i$, since $\mu_{1+(k-1)M}>\mu_{2+(k-1)M}>\dots>\mu_{kM}$. From \eqref{eq:reward_central}, we can get:
\begin{equation}
\label{eq:Rs_max}
\begin{aligned}
    R^{\text{max}}_{k} &\leq \sum_{j=k}^{L}\Bigg\{(1 - j\tau)\mu_{(j-1)M+1}\prod_{x=1}^{k-1}(1-\mu_{(x-1)M+1})\Bigg\}\\
    &\leq (L-k+1)\,\mu_{(k-1)M+1}.
\end{aligned}
\end{equation}
Recall that $\alpha = \frac{\mu_{\text{max}}}{\Delta_{\text{min}}}$, where $\mu_{\text{max}} = \underset{i}{\max}\,\mu_{i}$ and $\Delta_{\text{min}} = \underset{i < j}{\min}\,\Delta_{i,j}$. 
Combining Lemma \ref{lemma:Ts} and \eqref{eq:Rs_max} gives the upper bound of $R^{s}_{i,k}(T)$:
\begin{equation}
\label{eq:Rs(T)}
\begin{aligned}
    \mathbb{E}[R^{s}_{i,k}(T)] &\leq R^{\text{max}}_{k}\,\mathbb{E}[T^{s}_{i,k}]\\
    &\leq kM(L-k+1)(\frac{8\log T}{\Delta_{kM,i}^2} + 1 + \frac{\pi^2}{3})\mu_{1+(j-1)M}\\
    &\leq \alpha kM(L-k+1)\left [\frac{8\log T}{\Delta_{kM,i}} + (1 + \frac{\pi^2}{3})\Delta_{kM,i}\right ]
\end{aligned}
\end{equation}

Let us move to the discussion of $R^{a}_{i,k}(T)$. This part of the regret comes from the fact that, at the $k^{\text{th}}$ observation step, although players choose from the correct set of arms $\{(k-1)M+1, (k-1)M+2, \dots, kM\}$, there are $M!$ possible allocations, which might cause regret compared to the baseline policy. Now we need to consider the regret of putting arm $i$ into the wrong place within the correct observation step $k$, where $(k-1)M+1 \leq i \leq kM$. Denote $T^{a}_{i,k}$ as the number of times that arm $i$ appears in a wrong place at the correct observation step $k$. Under Algorithm \ref{alg:c-mp-obp}, we can bound $\mathbb{E}[T^{a}_{i,k}]$ for all $(k-1)M+1 \leq i \leq kM$.
\begin{lemma}
\label{lemma:Ta}
For all $(k-1)M+1 \leq i \leq kM$, under Algorithm \ref{alg:c-mp-obp}, $\mathbb{E}[T^{a}_{i,k}] \leq (i-1)\,(\frac{8\log T}{\Delta_{i-1,i}^2} + 1 + \frac{\pi^2}{3}) +(K-i)\,(\frac{8\log T}{\Delta_{i,i+1}^2} + 1 + \frac{\pi^2}{3})$.
\end{lemma}
\begin{proof}
Let us first consider that arm $i$ appears before its correct place and denote the number of times it happens as $T^{a-}_{i,k}$. Algorithm \ref{alg:c-mp-obp} sorts the UCB values of arms, so $T^{a-}_{i,k}$ is equal to the number of times that $\hat{\mu}_i(t)$, the UCB value of arm $i$, is at least the $i-1$ largest one in $\hat{\boldsymbol{\mu}}(t)$. In that case, at least one arm in the set $\{1, 2, \dots, i-1\}$ has smaller UCB value than $\hat{\mu}_i(t)$, since at most $i-2$ arms have larger UCB values than $\hat{\mu}_i(t)$. Thus, $T^{a-}_{i,k}$ can be bounded by the number of times that the minimum UCB value of arms $\{1, 2, \dots, i-1\}$ is less than $\hat{\mu}_i(t)$. On the other hand, if arm $i$ appears after its correct place, denote the number of times it happens as $T^{a+}_{i,k}$. In that case, at least one arm in the set $\{i+1, i+2, \dots, K\}$ has larger UCB value than $\hat{\mu}_i(t)$, since at most $K-i$ arms have smaller UCB values than $\hat{\mu}_i(t)$. Thus, $T^{a+}_{i,k}$ can be bounded by the number of times that the maximum UCB value of arms $\{i+1, i+2, \dots, K\}$ is larger than $\hat{\mu}_i(t)$. We can get:
\begin{equation}
\begin{aligned}
    T^{a-}_{i,k} &\leq \sum_{t=1}^{T} \mathbbm{1}\{\min_{1\leq j\leq i-1} \hat{\mu}_j(t) \leq \hat{\mu}_i(t)\}\\
    &\leq \sum_{j=1}^{i-1}\sum_{t=1}^{T} \mathbbm{1}\{\hat{\mu}_j(t) \leq \hat{\mu}_i(t)\}.
\end{aligned}
\end{equation}
\begin{equation}
\begin{aligned}
    T^{a+}_{i,k} &\leq \sum_{t=1}^{T} \mathbbm{1}\{\max_{i+1\leq j\leq K} \hat{\mu}_j(t) \geq \hat{\mu}_i(t)\}\\
    &\leq \sum_{j=i+1}^{K}\sum_{t=1}^{T} \mathbbm{1}\{\hat{\mu}_j(t) \geq \hat{\mu}_i(t)\}.
\end{aligned}
\end{equation}
Similar to Lemma \ref{lemma:Ts}, we can bound both terms $T^{a-}_{i,k}$ and $T^{a+}_{i,k}$, and $T^{a}_{i,k}$ should be less than their sum:
\begin{equation}
\begin{aligned}
    \mathbb{E}[T^{a}_{i,k}] &\leq
    \mathbb{E}[T^{a-}_{i,k}] + \mathbb{E}[T^{a+}_{i,k}]\\
    &\leq
    \mathbb{E}\left[\sum_{j=1}^{i-1}\sum_{t=1}^{T} \mathbbm{1}\{\hat{\mu}_j(t) \leq \hat{\mu}_i(t)\} +     \sum_{j=i+1}^{K}\sum_{t=1}^{T} \mathbbm{1}\{\hat{\mu}_j(t) \geq \hat{\mu}_i(t)\}\right]\\
    &\leq \sum_{j=1}^{i-1}\left\{ \frac{8\log T}{\Delta_{j,i}^2} + 1 + \frac{\pi^2}{3}\right\} + \sum_{j=i+1}^{K}\left\{ \frac{8\log T}{\Delta_{i,j}^2} + 1 + \frac{\pi^2}{3}\right\}\\
    &\leq (K-1)\,(\frac{8\log T}{\Delta_{\min}^2} + 1 + \frac{\pi^2}{3}),
\end{aligned}
\end{equation}
which concludes the proof.
\end{proof}
With Lemma \ref{lemma:Ta} and \eqref{eq:Rs_max}, we can write $R^{a}_{i,k}$ as:
\begin{equation}
\label{eq:Ra(T)}
\begin{aligned}
    \mathbb{E}[R^{a}_{i,k}(T)] \leq& R^{\text{max}}_{k}\,\mathbb{E}[T^{a}_{i,k}]\\
    \leq &(L-k+1)(i-1)(\frac{8\log T}{\Delta_{i-1,i}^2} + 1 + \frac{\pi^2}{3})\mu_{1+(j-1)M}\\
    +& (L-k+1)(K-i)(\frac{8\log T}{\Delta_{i,i+1}^2} + 1 + \frac{\pi^2}{3})\mu_{1+(j-1)M}\\
    \leq & c_{\mu} (L-k+1)(i-1)\left [\frac{8\log T}{\Delta_{i-1,i}} + (1 + \frac{\pi^2}{3})\Delta_{i-1,i}\right]\\
    +&c_{\mu} (L-k+1)(K-i)\left [\frac{8\log T}{\Delta_{i,i+1}} + (1 + \frac{\pi^2}{3})\Delta_{i,i+1}\right].
\end{aligned}
\end{equation}
Define $T_{\max}:= \frac{8\log T}{\Delta_{\min}} + (1 + \frac{\pi^2}{3})\Delta_{\max}$.
Finally, with \eqref{eq:reward_central}, \eqref{eq:Rs(T)} and \eqref{eq:Ra(T)}, we can bound $\mathbb{E}[R(T)]$:
\begin{equation}
\begin{aligned}
    \mathbb{E}[R(T)] \leq& \sum_{k=1}^{L} \left\{\sum_{i>kM}^{K}\mathbb{E}[R^{s}_{i,k}(T)]+\sum_{i=(k-1)M+1}^{kM} \mathbb{E}[R^{a}_{i,k}(T)]\right\}\\
    \leq& \sum_{k=1}^{L}\Bigg\{\sum_{i>kM}^{K} c_{\mu} kM(L-k+1) T_{\max} +\sum_{i=(k-1)M+1}^{kM}c_{\mu} (L-k+1)(K-1)T_{\max}\Bigg\}\\
    \leq& c_{\mu} L^2 K^2\,T_{\max} + \alpha L^2 M K\,T_{\max}\\
    \leq& c_{\mu} K^2(L^2+L)\,T_{\max}.
\end{aligned}
\end{equation}
\end{proof}

\section{Proof of Theorem \ref{theorem:distributed}}
\begin{proof}
In order to prove Theorem \ref{theorem:distributed}, we first consider the following lemma:
\begin{lemma}\label{lemma:loss-distributed}
\begin{equation}\label{eq:loss-initial}
    \mathbb{E}[\text{Loss}(T)] \le \mu_{\max}\mathbb{E}[\#~\text{of collisions}] + \sum_{k=1}^L\sum_{i>km}^K R_{i,k}^s(T)
\end{equation}
\end{lemma}

\label{app:theorem:distributed}
\begin{proof}
Here $R_{i,k}^s$ is as defined in \eqref{equ:Rs+Ra}. Lemma \ref{lemma:loss-distributed} essentially upper-bounds $\text{Loss}(T)$ by the maximum regret caused by collisions and the total regret due to observing arms in the wrong steps. 
Whenever there are collisions at any given round $t$, the expected loss of reward compared to any offline policy is no larger than the highest regret at $t$ over all users who encounter a collision, {\em i.e.,} every user gets zero reward in our policy while every user gets the highest reward in expectation in the offline policy. When there's no collision, the loss compared to any greedy policy is caused by observing arms in the wrong steps, {\em i.e.,} which is at most $\sum_{k=1}^L\sum_{i>km}^K R_{i,k}^s(T)$.
\end{proof}

To further upper-bound $ \mathbb{E}[\text{Loss}(T)]$, we proceed in the next lemma to upper-bound $\mathbb{E}[\#~\text{of collisions}]$ across all players. The basic idea of the proof is to consider the number of collisions in: (1) rounds where each player chooses from the correct list of arms in each observation step and (2) rounds in which there exists at least one player having at least one arm in the wrong step. We respectively call these (1) good phases (i.e., sequential rounds where the first condition is satisfied in each round) and (2) bad rounds. The term $K \binom{2M-1}{M}$ upper-bounds the number of collisions of each step in each good phase, and $M$ upper-bounds that in each bad round. Since the number of non-sequential good phases is no larger than the number of bad rounds plus one, the lemma follows. 
\begin{lemma}\label{lemma:collisions}
The total expected number of collisions is at most
\begin{equation}\label{eq:collision}
\left( K\binom{2M-1}{M} + M\right) \times \sum_{k=1}^L\sum_{i>kM}^K\mathbb{E}[T_{i,k}^s]
\end{equation}
\end{lemma}
\begin{proof}
It is easy to verify the total number of collisions over all bad rounds are at most $M$ times the total number of those rounds. Thus, in the following, we only need to consider the good phases. In a good phase, every user has the same (and also correct) set of arms to observe in each step $i$. We simply check how the $M$ users are ``assigned to'' the $M$ arms in each step $i$. We first consider a given round $t$ where every user encounters a collision in round $t-1$. 
In this case, each user will uniformly at random select one out of those $M$ arms in round $t$. We now consider the total number of distinct configurations of arms and users. Since in this lemma, we are calculating the number of collisions rather than the reward or regret of each user, we do not distinguish different users choosing the same arm. Thus, two configurations are distinct iff there exists at least one arm that has a different number of assigned users between these two configurations. This random process is equivalent to assigning $M$ balls into $M$ boxes which has a total of $\binom{2M-1}{M}$ distinct configurations \cite{wiki:absorbing}\cite{book:combinatorics}.  
Now we consider the cases where $\gamma$ out of M users (let $M>\gamma >0$ will continue to choose the same arms as in the previous round, since there was no collision in the previous round. Similarly, the number of distinct user-arm configurations is at most $\binom{2M-1-\gamma}{M-\gamma}$, which is smaller than $\binom{2M-1}{M}$. Since each user's decision is only dependent on his decision and outcome in the previous round, this random process of assigning users to arms over time is a Markov chain with at most $\binom{2M-1}{M}$ states. Moreover, it's easy to verify that once the process enters a state where users choose different arms in a given step, it will stay in this state, as long as the good phase hasn't transitioned to a bad round. Therefore, this stochastic process is an Absorbing Markov chain with an absorbing time no larger than $\binom{2M-1}{M}$ rounds \cite{distributed_jsac}. Thus, the total number of collisions of each step within each good phase is at most $M\binom{2M-1}{M}$. However, we have to consider an extreme case where for any given observation step $i$, it enters an absorbing state with a number of $\binom{2M-1}{M}$ rounds, but the chosen arms of all users are realized to be unavailable. Thus, all of them have to enter observation step $i+1$ and the process starts over from a possibly transient state. The worst case is that the above extreme case happens over all $K/M$ observation steps. Therefore, the maximum number of collisions in a good phase is at most $K\binom{2M-1}{M}$. Combining the total number of bad rounds with the number of collisions in each good phase and bad round respectively, the lemma follows. 
\end{proof}

Note that the multiplicative term in \eqref{eq:collision}, $\mathbb{E}[T_{i,k}^s]$, has been given in \eqref{eq:number-wrong-steps}. Putting \eqref{eq:collision} and \eqref{eq:Rs(T)} into \eqref{eq:loss-initial}, we get Theorem \ref{theorem:distributed}. While this loss bound is logarithmic in the number of rounds $T$, like the $O(\frac{K^3}{M} L\log(T))$ regret bound given in Theorem~\ref{theorem:multi} for the C-MP-OBP policy, it is combinatorial in $M$ instead of being polynomial in $K = LM$. The lack of coordination in the distributed setting introduces an additional cost from possible collisions.
\end{proof}

\section{Proof of Theorem \ref{theorem:distributed_adapt}}\label{app:theorem:distributed_adapt}
\label{appx-theorem:distributed_adapt}
\begin{proof}
\begin{lemma}\label{lemma:adapt-distributed}
The total expected regret,
\begin{equation} \label{eq:adapt-init}
    \mathbb{E}[R(T)] \leq \mu_{\max} \mathbb{E}[\#~\text{of collisions}]\notag + \sum_{k=1}^L \left(\sum_{i>km}^K R_{i,k}^s(T) + \sum_{i=1(k-1)M + 1}^{KM} R_{i,k}^a(T)\right)
\end{equation}
\end{lemma}
\begin{proof}
The total expected regret can be upper-bounded by the sum of the expected loss and the expected regret due to choosing the wrong arm from the right step over all users. Combining the proof in Theorem \ref{theorem:distributed} and \eqref{eq:Ra(T)}, this lemma follows.
\end{proof}

To further upper-bound $R(T)$, we upper-bound the expected number of collisions in the following lemma. 
\begin{lemma}\label{lemma:collision-adapt}
We have:
\begin{equation}\label{eq:collision-adapt}
        \mathbb{E}[\#~\text{of collisions}] 
    \le M \left( \binom{2M-1}{M} + 1\right)\notag
    \times \sum_{k=1}^L\left(\sum_{i>kM}^K\mathbb{E}[T_{i,k}^s] + \sum_{i=(k-1)M + 1}^{KM} \mathbb{E}[T_{i,k}^{a}]\right)
\end{equation}
\end{lemma}
\begin{proof}
Interestingly, the first term in \eqref{eq:collision-adapt} (the number of collisions in a good phase) is smaller than that of our fair strategy D-MP-OBP. This can be explained intuitively as follows. According to our D-MP-Adapt-OBP, the decisions of the steps $2, \cdots, L$ are determined by the decisions of step $1$ and $f(\cdot)$, given the reward estimations of all arms. Therefore, within a good phase, when the first step becomes collision-free, the following steps will all be collision-free. In this sense, the number of collisions will not increase with the number of observation steps. 
Consistent with the terminologies used in the proof of Lemma \ref{lemma:collisions}, we consider each round where there exists a user who either chooses an arm in the wrong observation step (a bad round) or chooses the wrong arm from the right observation step. The analysis for the collisions in the former event is the same as Lemma \ref{lemma:collisions}. The latter event can be divided into three cases: (1) in the first observation step, multiple users play the same arm; (2) in a later observation step $i>1$, two or more user choose an unavailable arm $j$ in step $i-1$, and they both choose arm $f(j, \{\hat{\mu_k}\}_{k=1}^{K})$ in step $i$; (3) in a later observation step $i>1$, the user has a different order of arms with at least one other user, {\em e.g.,} user 1 and user 2 are supposed to choose the arms in the second position and the third position respectively but they both choose arm 2 as user 2 mistakenly ranks arm 2 in the third position. In any one of the above three cases, there is at most one collision encountered by each user in each round. Now we consider the good phases in which users have the same (and correct) order of arms. For the first observation step, there are at most $M\binom{2M -1}{M}$ rounds before entering an absorbing state. 
Since the positions of arms to choose in each step $i>1$ are determined by the arms chosen in step $1$,
observing the arms in each observation step $i>1$ (only when the arms chosen in the previous observation arm sets are unavailable) does not transition the state from an absorbing state to a transient state. Thus, the total expected number of collisions in a good phase over all steps is still $M\binom{2M-1}{M}$, which does not increase with the number of observation steps. Putting the above together, the lemma follows. 
\end{proof}

Combining Lemma \ref{lemma:adapt-distributed}, Lemma \ref{lemma:collision-adapt}, \eqref{eq:Rs(T)}, and \eqref{eq:Ra(T)}, the theorem directly follows.
\end{proof}

\section{Effect of \texorpdfstring{$\mu$}{Lg}}
We would intuitively expect that increasing the average rewards $\mu$ would increase the reward gap with the random baseline: it is then more important to pre-observe ``good'' arms first, to avoid the extra costs from pre-observing occupied arms. We confirm this intuition in each of our three settings. However, increasing $\mu$ does not always increase the reward gap with the single-observation baseline. If $\mu$ is very low or very high, pre-observations are less valuable and the reward gap is relatively small. When the $\mu$ are small, the player would need to pre-observe several arms to find an available one, decreasing the final reward due to the cost of these pre-observations. When the $\mu$ are large, simply choosing the best arm is likely to yield a high reward, and the pre-observations would usually be unnecessary.

Figures~\ref{fig:single_mu_offline_2} and~\ref{fig:single_mu_random_2} show the reward gap of the single-player OBP-UCB with the single-observation and random algorithms for different average rewards $\mu$ and a fixed value of $\tau$. As discussed above, a larger $\mu$ increases the reward gap with the random baseline, and first increases but then decreases the reward gap with the single-observation baseline.
In Figures~\ref{fig:multi_mu_opt}--\ref{fig:dis_mu_random}, we plot the reward gap with respect to $\mu$ in the centralized and distributed multi-player settings. As in the single-player setting, an increase in $\mu$ increases the reward gap with the random baseline, but has a non-monotonic effect compared to the single-observation baseline.
\captionsetup[figure]{labelfont=bf}
\begin{figure*}
\centering
\begin{subfigure}[t]{0.45\linewidth}
	\includegraphics[width=\textwidth]{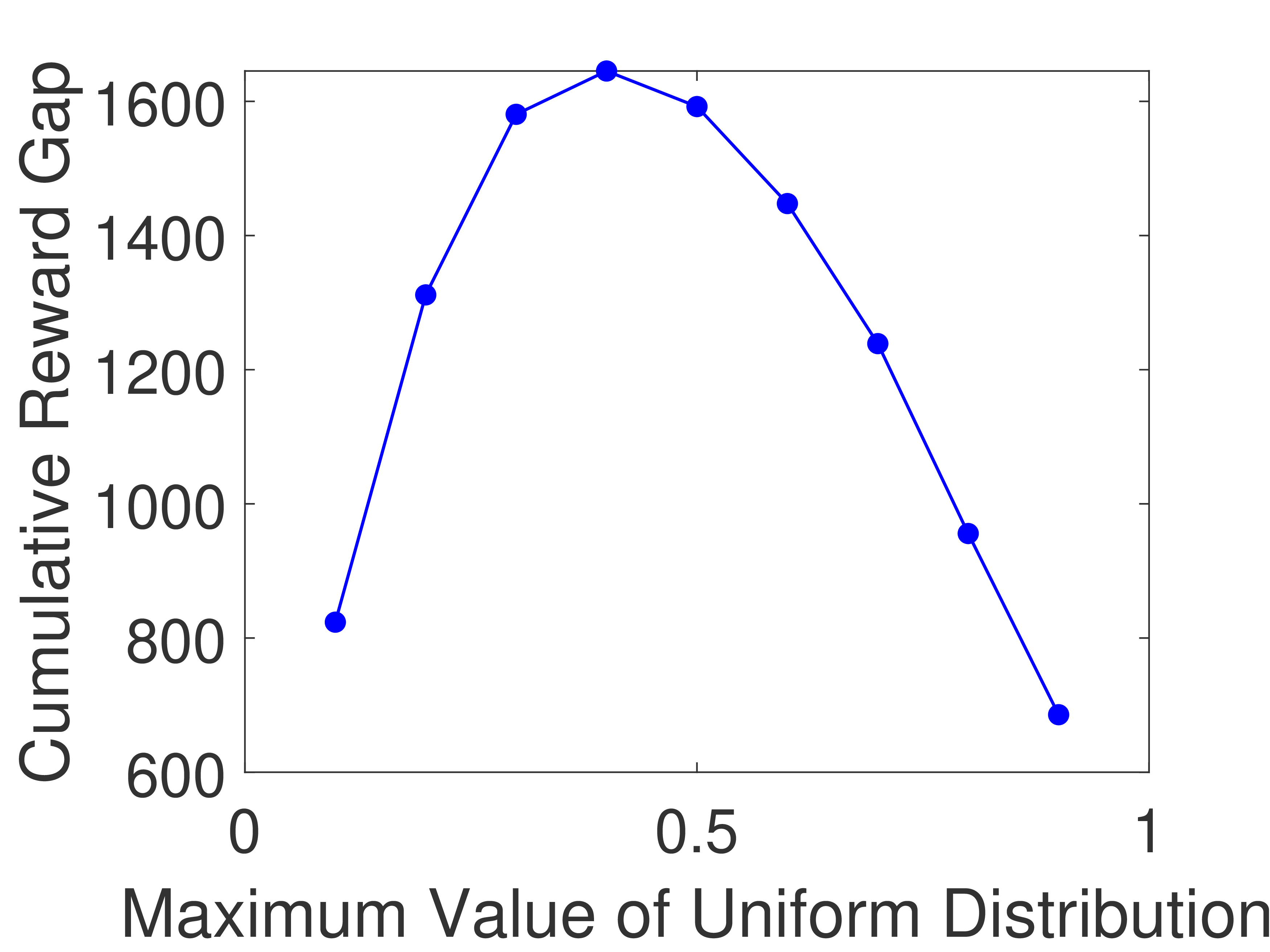}
	\caption{OBP-UCB: single-opt.}
	\label{fig:single_mu_offline_2}
\end{subfigure}
\begin{subfigure}[t]{0.45\linewidth}
	\includegraphics[width=\textwidth]{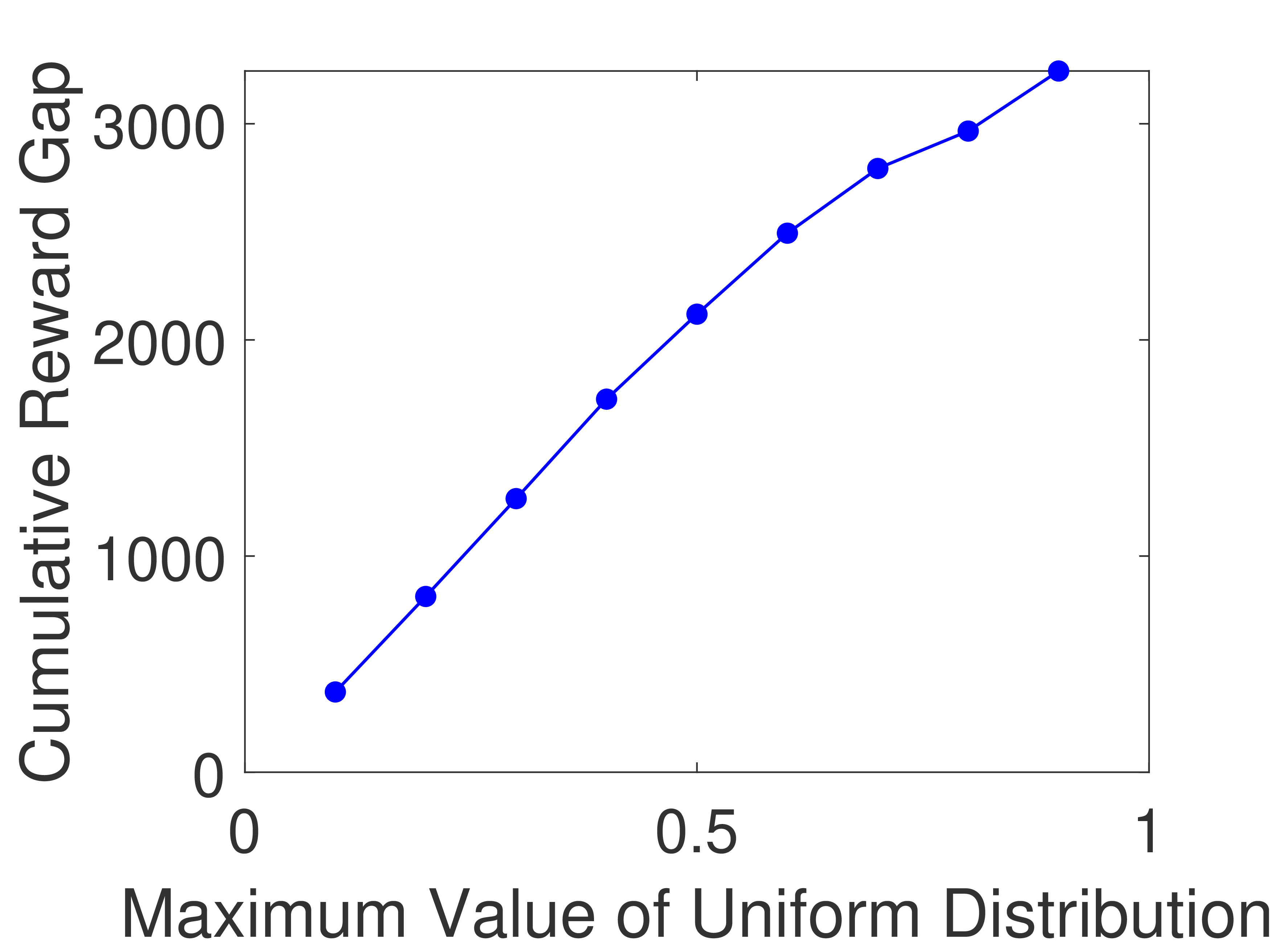}
	\caption{OBP-UCB: random.}
	\label{fig:single_mu_random_2}
\end{subfigure}
\newline
\begin{subfigure}[t]{0.45\linewidth}
	\includegraphics[width=\textwidth]{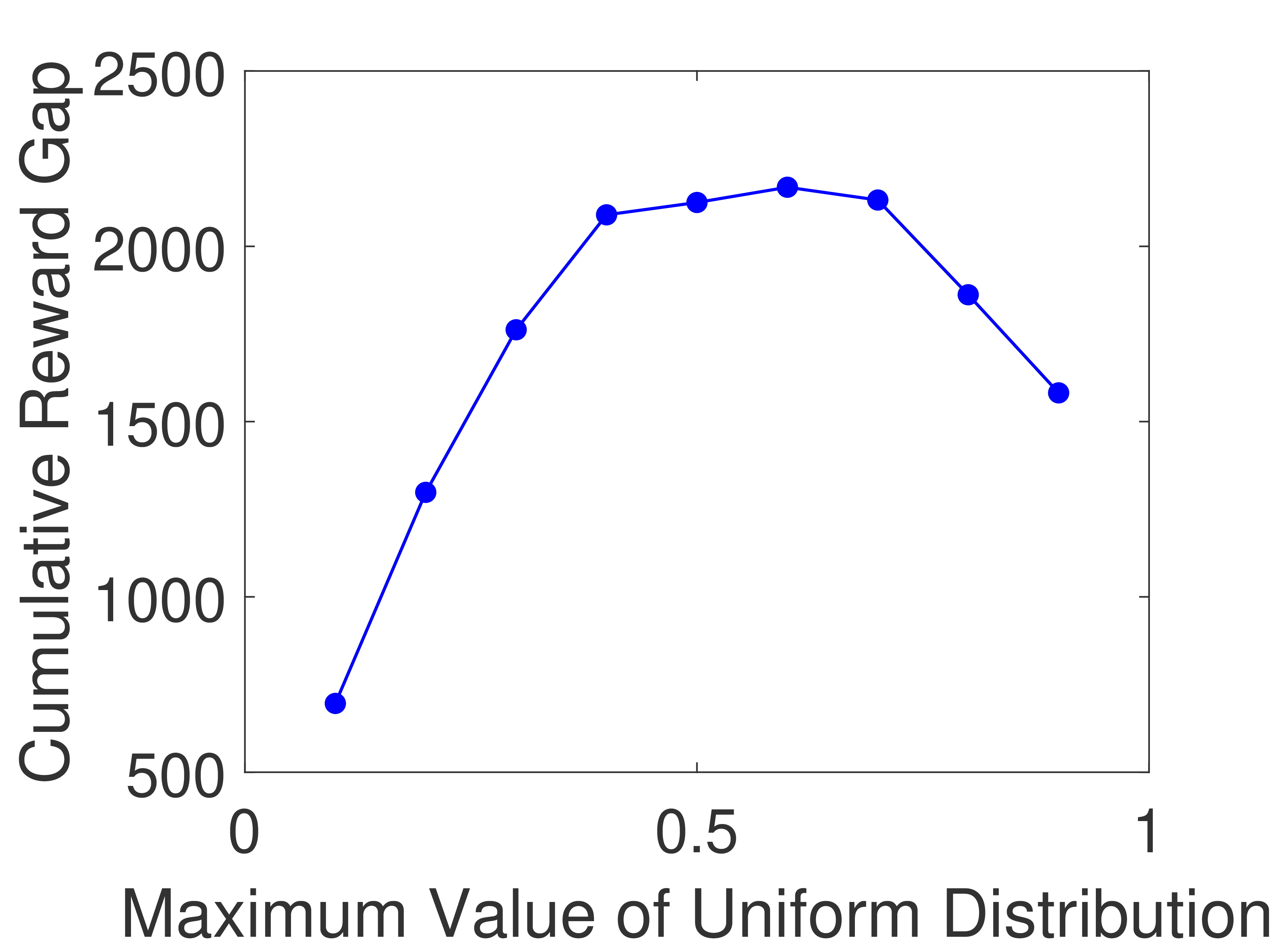}
	\caption{C-MP-OBP: single-opt.}
	\label{fig:multi_mu_opt}
\end{subfigure}
\begin{subfigure}[t]{0.45\linewidth}
    \centering
	\includegraphics[width=\textwidth]{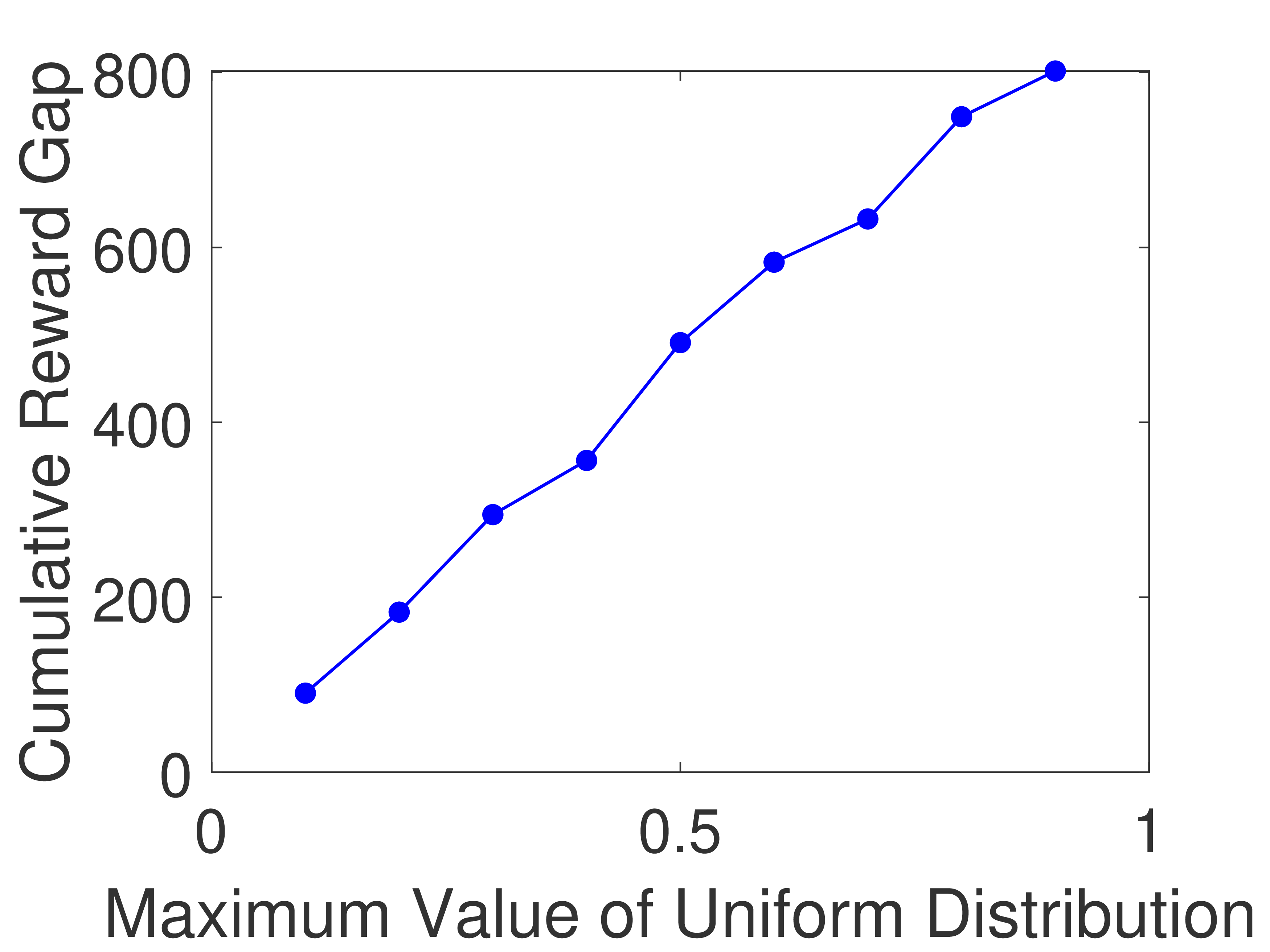}
	\caption{C-MP-OBP: random.}
	\label{fig:multi_mu_random}
\end{subfigure}
\newline
\begin{subfigure}[t]{0.45\linewidth}
	\includegraphics[width=\textwidth]{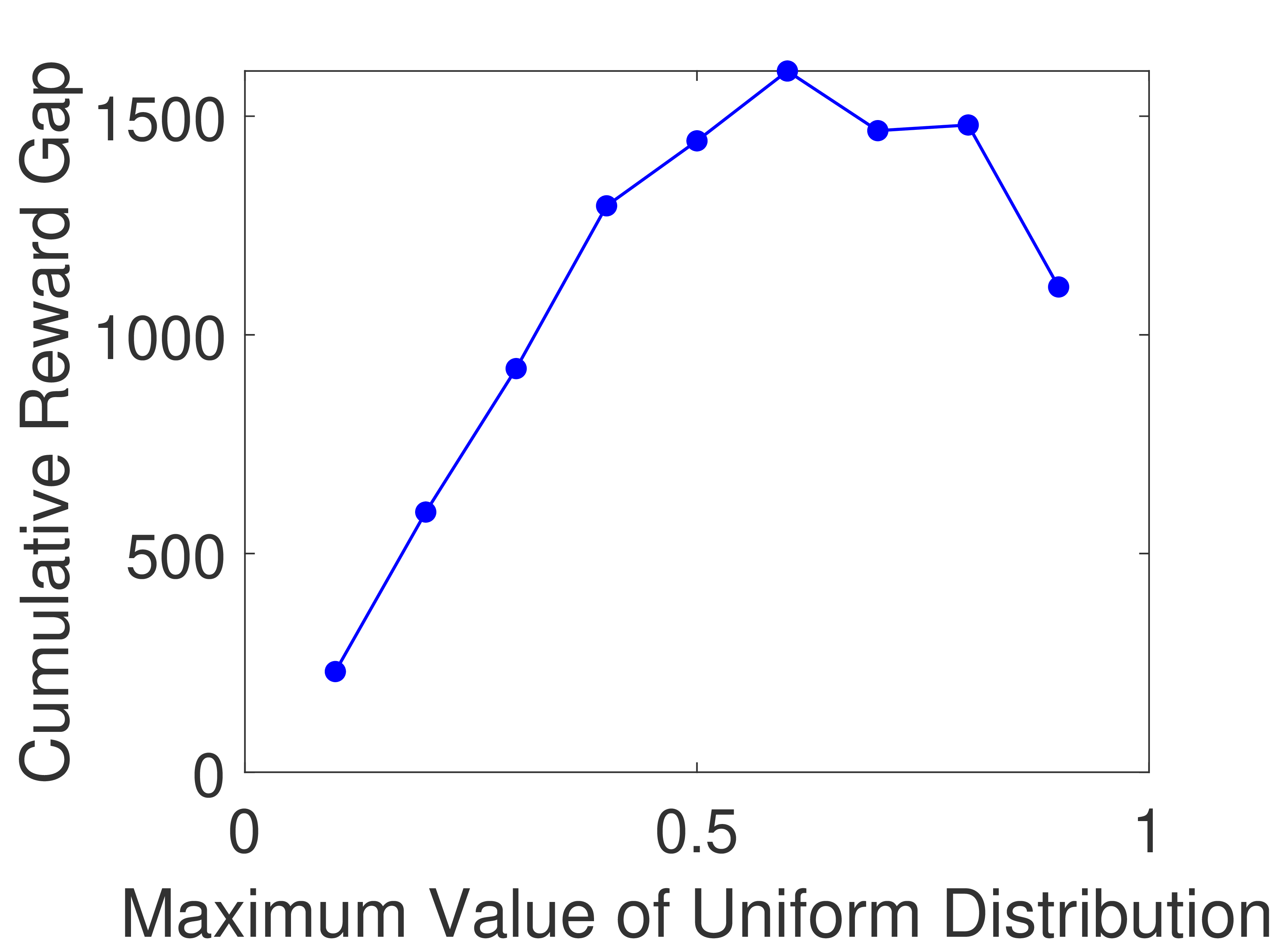}
	\caption{D-MP-OBP: single-opt.}
	\label{fig:dis_mu_opt}
\end{subfigure}
\begin{subfigure}[t]{0.45\linewidth}
    \centering
	\includegraphics[width=\textwidth]{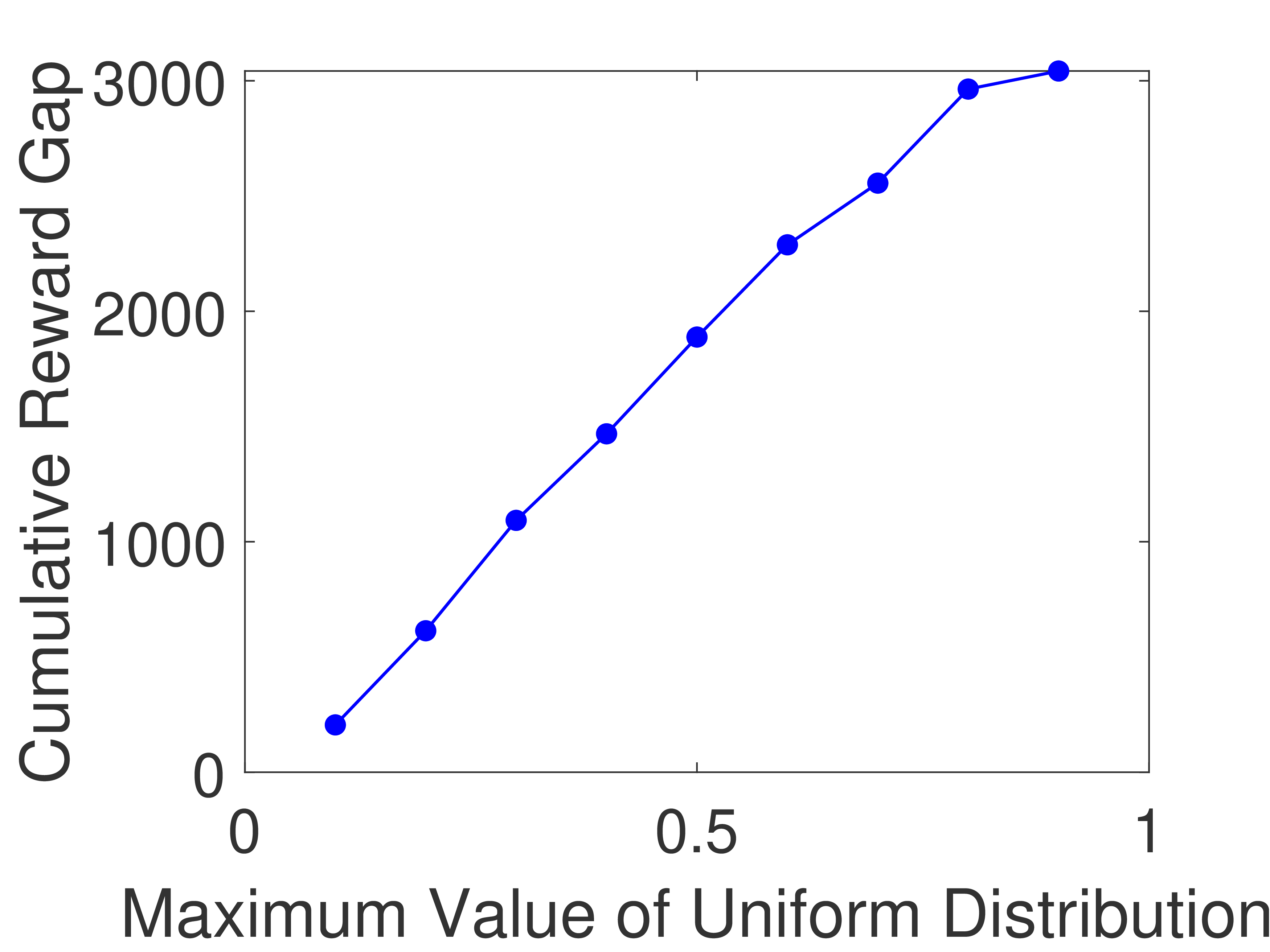}
	\caption{D-MP-OBP: random.}
	\label{fig:dis_mu_random}
\end{subfigure}
\caption{Average cumulative reward gaps in the single-player (OBP-UCB), centralized multi-player (C-MP-OBP), and distributed multi-player (D-MP-OBP) settings after 5000 rounds over 100 experiment repetitions with $\tau = 0.1$ and $K = 9$ arms and rewards uniformly drawn from $[0, x]$, where $x$ is the maximum value of the uniform distribution.}
\label{fig:mu_2}
\end{figure*}

\end{document}